\documentclass[11pt]{amsart}
\usepackage{fullpage}
\usepackage{amsmath, amssymb, amsthm, amsfonts}
\usepackage[foot]{amsaddr}
\usepackage{mathtools}
\usepackage[inline]{enumitem}

\usepackage{graphicx}
\graphicspath{{./figs/}} %note trailing /
\usepackage{algpseudocode}
\usepackage{algorithm}

\usepackage{epsfig}
\usepackage{enumitem}

\newcommand{\ie}{{\it i.e.}}
\newcommand{\eg}{{\it e.g.}}

\DeclareMathOperator*{\argmin}{arg\,min}

\usepackage{color}
\usepackage[colorlinks=true,linkcolor=red,citecolor=blue,urlcolor=cyan]{hyperref}
\usepackage{cleveref}% load last

\newtheorem{thm}{Theorem}[section]
\newtheorem{lem}[thm]{Lemma}
\newtheorem{assum}[thm]{Assumption}
\newtheorem{rem}[thm]{Remark}

\usepackage[backend=biber,date=year,giveninits=true,sorting=nyt,natbib=true,maxcitenames=2,maxbibnames=10,url=false,doi=true,backref=false]{biblatex}
\renewbibmacro{in:}{\ifentrytype{article}{}{\printtext{\bibstring{in}\intitlepunct}}}

\begin{document}
\title{A dynamical systems based \\ framework for  dimension reduction}
\thanks{R. Yoon and B. Osting acknowledge partial support from NSF DMS 17-52202.}

\author{Ryeongkyung Yoon}
\address{Department of Mathematics, University of Utah, Salt Lake City, UT}
\email{\{rkyoon,osting\}@math.utah.edu}      

\author{Braxton Osting}
%\address{Department of Mathematics, University of Utah, Salt Lake City, UT}
%\email{osting@math.utah.edu}



\date{\today}

\begin{abstract} 
We propose a novel framework for learning a low-dimensional representation of data based on nonlinear dynamical systems, which we call \emph{dynamical dimension reduction} (DDR). In the DDR model, each point is evolved via a nonlinear flow towards a lower-dimensional subspace; the projection onto the subspace gives the low-dimensional embedding. Training the model involves identifying the nonlinear flow and the subspace. Following the equation discovery method, we represent the vector field that defines the flow using a linear combination of dictionary elements, where each element is a pre-specified linear/nonlinear candidate function. A regularization term for the average total kinetic energy is also introduced and motivated by optimal transport theory. We prove that the resulting optimization problem is well-posed and establish several  properties of the DDR method. We also show how the DDR method can be trained using a gradient-based optimization method, where the gradients are computed using the adjoint method from optimal control theory. The DDR method is implemented and  compared on synthetic and example datasets to other dimension reductions methods, including PCA, t-SNE, and Umap. 
\end{abstract}

\maketitle

\section{Introduction} 
There has been a growing effort to develop dimension reduction techniques, which find an embedding of high-dimensional data into a meaningful representation space of smaller dimension. 
Such methods can be applied to a variety of machine learning tasks such as data visualization, outlier detection, and clustering.  
The most traditional approach is the principal component analysis (PCA) \cite{PCA},  which determines the linear subspace of a fixed dimension that captures the most variance in the data. PCA is a very practical method for extracting characteristic features in massive datasets and has a relatively small computational cost. 
However, as a linear method, PCA may not perform well in learning complex or nonlinear structures in data. 
In particular, since PCA equally weights all pairwise distances within the data, it favors preserving global structure over local structure and it can lose local information within a dataset. 

To overcome these limitations, a variety of nonlinear methods have been proposed, including $t$-distributed stochastic neighbor embedding (t-SNE) \cite{tsne},  Uniform manifold approximation and projection (Umap) \cite{umap}, kernel PCA, spectral embeddings, autoencoders \cite{auto,VAE}.
In particular, an autoencoder learns an  
\emph{encoder} $\mathcal{E} \colon \mathbb R^d \to \mathbb R^k$ as well as a \emph{decoder} $\mathcal{D} \colon \mathbb R^k \to \mathbb R^d$ so that the composition $\mathcal D \circ \mathcal E$ approximates the identity when applied to the data. 
The encoding step can be viewed as a nonlinear dimension reduction mapping and the reduced-dimension space is referred to as the \emph{latent space}; see \cref{subsec:Generative} for more details. 

Our \emph{goal} in this paper will be do develop a dimension reduction method based on nonlinear dynamical systems. The data is evolved via a nonlinear flow towards a lower-dimensional subspace, the latent space; the projection onto the latent space gives the low-dimensional embedding of the data. Our loss function for training the model is a modification of the loss function for an autoencoder; it penalizes the projection residual for the latent space. 
In the past few years, there has been significant research on the connections between dynamical systems and (residual) neural networks \cite{ChenRBD18,Regularity,ffjord,Haber}; we will discuss these related works and describe how these ideas differ from our model in \cref{subsec:NODE}.

\medskip

We begin with a motivating example that helps illustrate (and prompts many questions) how dynamical systems might be used to develop a low-dimensional representation of data.

\subsubsection*{Motivating example} \label{s:IntroExample}
Suppose we have data $X = [ x_1 \mid \cdots \mid x_N] \in \mathbb R^{d\times N}$ with $N > d$ with singular value decomposition, 
$X = U \Sigma V^\ast$, where
the singular values are arranged in decreasing order, \ie, $\sigma_1 \geq \cdots \geq \sigma_d$. 
Let $U_k \in \mathbb R^{d\times k}$ be the first $k$ columns of $U$ and
 $U_{-k} \in \mathbb R^{d\times d-k}$ be the remaining $d-k$ columns of $U$, \ie, $U = [U_k, U_{-k}]$.
The low dimensional representation of this data using PCA would be $\{U_k^\ast x_i\}_{i \in [N]} \subset \mathbb R^k$ with mean squared residual error  $\frac{1}{N}\sum_{i \in [N]}  \| x_i - U_k U_k^\ast x_i\|_2^2 = \frac{1}{N} \sum_{j = k+1}^d \sigma_j^2(X) = MSE_{PCA}$.  
 Alternatively, we can construct a linear dynamical system that approximately gives this low dimensional representation. Define the matrix $A_\varepsilon \in \mathbb R^{d \times d}$ by $A_\varepsilon = \frac{1}{T} U \textrm{diag}(0,\ldots, 0, \log \varepsilon, \ldots, \log \varepsilon ) U^\ast$, where 0 is repeated $k$ times and $\log \varepsilon $ is repeated $d-k$ times. We then consider the initial value problem for each $i \in [N]$, 
\begin{align*}
& \dot h_i(t)  = A_\varepsilon h_i(t) \\
& h_i(0)  = x_i,
\end{align*}
where $x_i$ denotes the $i$-th column of $X$. The solution is given by $h_i(t) = e^{A_\varepsilon t} x_i$, $i\in [N]$, so that at time $t=T$, we have 
$$
h_i(T) = e^{A_\varepsilon T} x_i
= U  \textrm{diag}(1,\ldots, 1,  \varepsilon, \ldots,  \varepsilon ) U^\ast x_i 
= \left( U_kU_k^\ast + \varepsilon U_{-k}U_{-k}^\ast \right) x_i
$$
Each data point $x_i$ evolves in $\mathbb R^d$ towards a low dimensional subspace $h_i(T)$, with  
$$
\|  h_i(T) - U_k U_k^\ast x_i \|_2^2 = \varepsilon^2  \|  U_{-k}U_{-k}^\ast x_i  \|_2^2 
$$
This implies that 
$$
\frac{1}{N}
\sum_{i \in [N]} \|  h_i(T) - U_k U_k^\ast x_i \|_2^2 
= \frac{\varepsilon^2 }{N}  \sum_{j = k+1}^d \sigma_j^2(X)
= \varepsilon^2 MSE_{PCA}
$$
In other words, the mean squared distance between the solution at time $t=T$ and the best $k$-dimensional representation of the data in $\mathbb R^d$ is $O(\mathbb \varepsilon^2) $.

\subsubsection*{Our framework}
In this paper, we formulate a method, which we call the \emph{dynamical dimension reduction} (DDR) model, that generalizes the above example in several ways: 
\begin{enumerate*}[label=(\roman*)]
\item We allow the right-hand side (RHS) of the dynamical system to be a nonlinear vector field. 
\item We formulate an optimization problem that finds a RHS which evolves the data towards a low dimensional representation. 
\item We also introduce a regularization term in the objective function based on the mean total kinetic energy of the trajectories, which preserves the local and global structure of the data. 
\end{enumerate*}

For each data point $x_i \in \mathbb R^d$, $i \in [N]$, we introduce a time-dependent hidden variable $h_i \in \mathbb R^d$ that is governed by the dynamical system 
\begin{align*}
    \frac{dh_i}{dt} &= \Phi(h_i; \beta)\\
    h_i(0) & = x_i,
\end{align*}
where the vector field, $\Phi(\cdot;\beta)$,  is parametrized by $\beta$. 
A description of the parameterization of $\beta \mapsto \Phi(\cdot;\beta)$ using a dictionary of linear and nonlinear terms will be given in \cref{sec3}.
For fixed final time $T>0$, the low-dimensional embedding $x_i \xmapsto{\mathcal{E}} y_i$ is  defined using the solution at time $t=T$; the data point $x_i = h_i(0)$ is encoded in a lower ($k<d$) dimensional space via $y_i = Q h_i(T)$, where $Q \in \mathbb R^{k \times d}$ is a matrix with orthonormal rows. 
Training the network then involves learning the parameters $\beta$ and $Q$. To achieve this goal, we introduce an objective function of the form 
\[
J(\beta, Q) := \frac{1}{N}\sum_{i \in [N]}  \| h_i(T)-Q^\ast Q h_i(T)\|^2 +  \mu R(h_i; \beta )
\]
and minimize $J$ over an appropriate set of parameters. The first term in the objective is the mean squared projection residual  and encourages the dynamical system to flatten the data as time evolves. The second term is a regularization term that will be used to enforce smoothness on the vector field $\Phi$.  In particular, we choose a regularization term of the form 
\[
R(h_i;\beta) = \int_0^T  \| \Phi(h_i;\beta) \|^2~dt,
\]
which is a measure of the the kinetic energy of the trajectory $h_i(t)$, $t \in [0,T]$. This regularization term can also be interpreted in terms of optimal transport theory and the 
Wasserstein distance between the data distribution at initial and final times \cite{OT, OT_BB}; see further discussion in \cref{subsec:OT}.

\subsubsection*{Overview of results} 
In \cref{sec3}, we formulate the optimization problem in more detail, including the proof of several theoretical results about the DDR method. We prove the existence of a minimizer of the proposed optimization problem (see \cref{subsec:existence}). We also show that the gradient of the objective function $J(\beta,Q)$ with respect to the parameters can be efficiently computed using the adjoint method from optimal control theory (\cref{thm:adjoint}). We introduce an alternating optimization method, described in \cref{s:Qsubproblem}, that alternatively updates $\beta$ and $Q$. We show that the $Q$-subproblem can be explicitly solved in terms of the singular value decomposition. In \cref{sec4}, we present a few properties of the DDR model. We prove the stability/generalizability of the embedding  $x_i \xmapsto{\mathcal{E}} y_i$ (\cref{thm:stability}).
We also revisit the motivating linear example discussed above and reproduce the result of PCA based on the DDR framework (\cref{lem:opt_linear}). 

In \cref{s:GenModel}, we extend the DDR method as a generative model by approximating the decoder $y_i \xmapsto{\mathcal{D}} x_i$ based on the time-reversal of the learned dynamical system (see also \cref{t:Decoder}). 

Finally, in \cref{sec:experiments}, we describe the results of several numerical experiments that examine the performance of the DDR method on a variety of synthetic and example datasets. In these experiments, the DDR method achieves a competitive lower dimensional embedding with respect to other methods; PCA, t-SNE, and Umap. We illustrate that nonlinearity in the vector field of the dynamical system governing the time-evolution of a given data increases the representability/expressibility of the dimension reduction mapping. We also exhibit how stable the encoder is to the noise in the dataset and illustrate the DDR-based generative model. 

We conclude in \cref{s:disc} with a discussion of the DDR method and ideas for several future directions.

\clearpage 

\section{Background and related work}\label{sec2}
In this section, we review some related work that motivates the framework of the DDR method: autoencoders, neural ODEs, equation discovery, and optimal transportation. 

\subsection{Autoencoders}\label{subsec:Generative}
An \emph{autoencoder} \cite{auto} is comprised of two neural networks:  an \emph{encoder} $\mathcal E \colon \mathbb R^d \to\mathbb R^k$ and a \emph{decoder} $\mathcal D \colon \mathbb R^k \to \mathbb R^d$. The networks are trained so that the composition, $\mathcal D \circ \mathcal E$, approximates the identity on the data in terms of the mean residual error, 
 $\frac{1}{N} \sum_{i\in[N]} \|x_i - \mathcal D(\mathcal E(x_i))\|^2$. 
 Since $k\ll d$, we can interpret an autoencoder passing the data through a bottleneck structure while preserving as much information as possible. The encoder can be viewed as a nonlinear dimension reduction mapping into the latent space, $\mathbb R^k$. 
 
 However, if the capacity of the model is very large (\ie, there is a large degree of freedom in the autoencoder), it could fail to learn meaningful features in the data manifold and  to achieve the generative purpose \cite{Goodfellow}. 
To prevent this from happening, there are several ways to regularize an autoencoder, including
 (i) a penalizing regularity term can be introduced to promote sparsity in the model weights and reduce the sensitivity of the model with respect to given data, or
 (ii) reinterpreting the model based on variational inference, referred to as variational autoencoders (VAEs)  \cite{VAE}. 
VAEs estimate a posterior conditional probability of the encoder from a known prior distribution on the latent vector. 

 \subsection{Neural ODEs}\label{subsec:NODE}
In \cite{Haber},  the connection between residual neural networks with infinite depth and their continuum limit---a dynamical system---was developed. This idea was extended by \cite{ChenRBD18} and the framework was named Neural ODE (NODE). 
Here, for an input datapoint $x \in \mathbb R^d$, we introduce a time-dependent hidden variable  $h(t)\in\mathbb R^d$ that is governed by the dynamical system 
\begin{subequations}
  \begin{align}\label{eq:NODE}
     \frac{dh(t)}{dt} &= \Phi(h(t),t;\Theta),\\
     h(0) & = x.
 \end{align}
\end{subequations}

The underlying vector field $\Phi$ is represented using a  feedforward neural network \cite{ChenRBD18,Haber}.  Instead of backpropagation, the NODE is trained using the adjoint method from optimal control theory. 
Recently, the NODE framework has been further developed and extended in a variety of ways, including  
(i) demonstrating the NODE architecture improves accuracy and stability of the model \cite{antisym,ChenRBD18}, 
(ii) generalizing the network by allowing time dependence in the parameters  \cite{AdapNODE}, and 
(iii) modifying the mathematical framework of NODE via statistical process \cite{Gaussian} or partial differential equation \cite{PDE}.

Recently, NODE models have also been used to study unsupervised learning problems, particularly density estimation.  
\cite{ChenRBD18,Regularity,ffjord} have developed a novel and easily computed framework for a continuous normalizing flow that minimizes the difference in log densities for the data $x$ and hidden variable $h$. 
In particular,  \cite{Regularity} introduces a well-conditioned ODE-based model by imposing regularity via optimal transportation theory. 
However, this framework is not applicable for dimension reduction because the dimension of the latent space should have the same dimension as the data. 
Motivated by VAEs,  \cite{ChenRBD18,NVAE} proposes time-invariant generative models for time series. However, it doesn't completely rely on the NODE model because the data is encoded by a recurrent neural network (RNN) whereas latent vectors are decoded by NODE.

\subsection{Equation Discovery}\label{subsec:ED}
Another method to parameterize a vector field $\Phi$ is to use the equation discovery method introduced by \cite{EqnDiscovery}. 
In contrast to NODE, the equation discovery method writes the vector field $\Phi$ as a linear combination of dictionary functions,
 \begin{equation*}
 \Phi= \beta \Xi(h(t)), 
 \end{equation*}
where $\Xi$ consists of pre-specified candidate functions and $\beta$ is a matrix of coefficients to be determined. 
Equation discovery has primarily been applied to learn the underlying equations that describe a physical system from measured data. To encourage sparsity on the representation of $\Phi$ in these applications, \cite{EqnDiscovery} proposes the Sparse Identification of Nonlinear Dynamics (SINDy) method, which uses iterative thresholds least-squares methods. This method was proved to be convergent in \cite{SINDy}. We recently employed equation discovery methods to develop a non-autonomous equation discovery method (NAED) for the time signal classification problem \cite{NAED}.

\subsection{Optimal transportation theory and the Wasserstein metric}\label{subsec:OT}
Here, we briefly recall some concepts from optimal transportation theory that help motivate our choice of regularization function. For simplicity, we ignore technical details and refer to \cite{OT} for a more rigorous discussion. The squared \emph{2-Wasserstein distance} between probability measures $\mu_0,\mu_T \in \mathcal{P}(\mathbb R^d)$ can be written
\begin{equation}
\label{eq:classicOT}
d^2_W(\mu_0,\mu_T) =   \inf_{P_\# \mu_0 = \mu_T}  \ \  \int \|x-P(x)\|^2 d\mu_0(x) ,
\end{equation}
Here, $P \colon \mathbb R^d \to  \mathbb R^d$ is a \emph{transportation plan} and the pushfoward constraint ($P_\# \mu_0 = \mu_T$) means that $\mu_T(A) = \mu_0(P^{-1}(A))$ for any set $A \subset \mathbb R^n$. This constraint can be interpreted that a transportation plan $P$ rearranges the density corresponding to the measure $\mu_0$ into the density corresponding to measure $\mu_T$. Eq. \eqref{eq:classicOT} is known as the \emph{Monge formulation} of the 2-Wasserstein distance. 

There is also an equivalent dynamical formulation of the Wasserstein metric due to Benamou and Brenier \cite{OT_BB}. Here we think about \emph{continuously} transporting  mass from $\mu_0$ to $\mu_T$. We introduce a family of measures $\mu_t\in \mathcal{P}(\mathbb R^d)$, $t \in [0,T]$ and abuse notation by also denoting their densities by $\mu_t$, $t \in [0,T]$. The Benamou-Brenier formulation is then to find the time-dependent velocity field $\Phi \colon \mathbb R^d \times [0,T] \to \mathbb R^d$ so that when the density evolves according to the \emph{continuity equation}, the \emph{action} is minimized: 
\begin{subequations}\label{eq:OT_BB}
 \begin{align}
\label{eq:OT_BB_min}
d^2_W(\mu_0,\mu_T)  =  \inf_{\mu_t, \Phi} \quad  &  T \cdot \int_0^T \int \|\Phi(x(t))\|^2 d\mu_t(x)dt \\
\label{eq:OT_BB_cont}
\text{s.t.}  \quad & \frac{\partial \mu_t}{\partial t} + \nabla\cdot(\mu_t \Phi) = 0 \\
& \mu_{t=0}=\mu_0, \quad  \mu_{t=T} = \mu_T. 
\end{align}
\end{subequations}
In particular, if $P$ is the optimal transportation map in the Monge formulation \eqref{eq:classicOT} and we define $P_t = \frac{T-t}{T} I + \frac{t}{T} P$, the optimal solution to \eqref{eq:OT_BB} is given by $\mu_t = (P_t)_\# \mu_0$.

We now consider two pointsets $\{x_i\}_{i \in [N]} \subset \mathbb R^d$ and $\{\tilde{x}_i\}_{i \in [N]}\subset \mathbb R^d$ with the same cardinality and their corresponding empirical distributions 
$$ 
\mu_0(x) = \frac{1}{N} \sum_{i \in [N]} \delta(x-x_i) 
\quad  \textrm{and} \quad 
\mu_T(x) = \frac{1}{N} \sum_{i \in [N]} \delta(x-\tilde{x}_i). 
$$
In this case, the Benamou-Brenier formulation reduces to finding trajectories $\{x_i(t)\} \subset \mathbb R^d$,  $i \in [N]$,  $t \in [0,T]$ and the time-dependent velocity field $\Phi \colon \mathbb R^d \times [0,T] \to \mathbb R^d$ satisfying
\begin{subequations} \label{eq:OT_BBL}
\begin{align}
d^2_W(\mu_0,\mu_T)  =  \inf_{\Phi, x_i(t)} \quad  &  \frac{T}{N} \sum_{i \in [N]} \int_0^T \|\Phi(x_i(t)) \|^2 dt \\
\text{s.t.}  \quad & \frac{d x_i}{dt} =  \Phi(x_i(t)), \quad i \in [N], \ t \in [0,T]  \\
& x_i(t=0) = x_i, \ x_i(t=T) = \tilde{x}_i. 
\end{align}
\end{subequations}
This can be viewed as a \emph{Lagrangian perspective} for the Benamou-Brenier formulation while \eqref{eq:OT_BB} is the \emph{Eulerian perspective}. If we re-enumerate the points so that $\tilde{x}_i = P(x_i)$, where $P$ is the optimal transportation plan in the Monge formulation, then the optimal trajectories are simply given by 
$x_i(t) = \left( 1 - \frac{t}{T}\right) x_i + \frac{t}{T} \tilde{x}_i$
and the optimal cost is $d^2_W(\mu_0,\mu_T) = \frac{1}{N} \sum_{i \in [N]} \| x_i - \tilde{x}_i \|^2$. 
That is, the velocity field with the smallest action simply linearly transports each point from its initial to final position at a constant speed.

\section{Dynamical Dimension Reduction}
\label{sec3}
In this section, we 
formulate our proposed dynamical dimension reduction model (\cref{subsec:main}), 
prove the well-posedness of the model (\cref{subsec:existence}), and 
describe a gradient-based optimization method for training (\cref{subsec:adjoint,s:Qsubproblem,s:algorithm}).  

\subsection{Dynamical Dimension Reduction Model}\label{subsec:main}
Let the data $x_i\in \mathbb R^{d}$, $\forall i \in [N]$ be given. We propagate the data to a lower dimensional subspace using the solution to a dynamical system, where the solution is initialized at the data and at a fixed terminal time $T$, the solution lies in (or very near) the low dimensional subspace. To this end, we define hidden variables $h_i \colon [0,T]\to \mathbb R^d$ for $i \in [N]$ that describe the trajectories of each data point and satisfy the dynamical system,
\begin{subequations} 
\label{eq:ODE}
\begin{align}
    \frac{d }{dt} h_i(t)& = \Phi\left(h_i(t) ;\beta \right), \quad \forall i \in [N] \label{eq:ODE_main} \\
 h_i(0) &= x_i\label{eq:ODE_ini}.
\end{align}
\end{subequations}
Here $\beta\in\mathbb R^{d\times d_n}$ is matrix used to parameterize the vector field $\Phi\colon\mathbb R^d \to \mathbb R^d $.  The solution to \eqref{eq:ODE} at time $T$ is used to define a low dimensional representation $y_i\in \mathbb R^{k}$ with $k\ll d$ as 
\begin{equation}\label{eq:low_embedding}
    y_i = Qh_i(T), \qquad \forall i \in [N],
\end{equation}
where 
\begin{equation}
\label{e:Ok}
Q\in \mathcal O_k := \{Q \in \mathbb R^{k \times d} \colon QQ^\ast  = I_k\} 
\end{equation}
is a matrix with orthonormal rows. 
As described further below, the parameters in this model, $\beta\in\mathbb R^{d\times d_n}$ and $Q \in \mathcal O_k$, will be optimized (a.k.a. trained) as to obtain a low-dimensional embedding of the data. 
We will refer to this mapping $\mathcal{E} \colon \mathbb R^d \to \mathbb R^k$ that assigns
$ x_i \xmapsto{\mathcal{E}} y_i$ 
as the \emph{dynamical dimension reduction (DDR) embedding}. 

We have chosen an autonomous vector field, $\Phi$, on the right hand side of \eqref{eq:ODE_main} in this work for simplicity, however a non-autonomous vector field could also be used.
We represent the vector field $\Phi$ using a dictionary of functions, as in the equation discovery method described in \cref{subsec:ED}. 
In our model,  $\Phi$ is parameterized by  
\begin{equation}\label{eq:rhs}
    \Phi(h_i;\beta) = \beta \Xi(h_i),\qquad \forall i \in [N],
\end{equation}
for a pre-specified dictionary  $\Xi(h_i)=[\xi_1(h_i),  \cdots,\xi_{d_n}(h_i)] \in \mathbb R^{ d_n}$ that consists of  candidate functions  $\xi_\ell\colon \mathbb R^{d} \to \mathbb R,$ for $\ell \in [d_n]$. There is tremendous freedom in the choice of dictionary which, in turn, determines the \emph{representability} or \emph{expressiveness} of our model. A key attribute of our method will be to choose dictionary elements which are \emph{nonlinear}; if only linear dictionary elements are chosen, the DDR embedding residual error can only be as good as the PCA embedding. For example, in the implementation discussed in Section \ref{sec:experiments}, we utilize   multivariate polynomials with degree $\leq 3$ as dictionary elements. However, with the introduction of nonlinear dictionary elements, we must consider whether, for each data point $x_i$, there exists a unique solution to the governing ODEs \eqref{eq:ODE} on the time interval $[0,T]$. The following theorem recalls sufficient conditions to guarantee the existence and uniqueness of a solution to  \eqref{eq:ODE} depending on a choice of a dictionary. Its proof relies on a standard existence/uniqueness argument in the theory of ordinary differential equations (see, \eg, \cite[Theorem 3.2]{ODE}).

\begin{thm}\label{thm:ODE_existence}
Let $K \subset \mathbb R^d$ be a compact set that contains every data point $x_i$, $i \in [N]$ and define 
$$
 r_K := \sup_{y \in K}\|y-x_i\|_2,\qquad \forall i \in [N].
$$
Suppose every $\xi_\ell \colon \mathbb R^d \to \mathbb R$ is Lipschitz continuous on $K$ with  Lipschitz constant $\mathcal L$, \ie,~for every $h_1,h_2 \in K$, 
\[
|\xi_\ell(h_1) - \xi_\ell(h_2)| \leq \mathcal L \|h_1-h_2\|_2, \qquad \forall \ell \in [d_n].
\]
Furthermore, assume $\max_{x\in K} |\xi_\ell(x)| \leq M$. 
Then the initial value problem  \eqref{eq:ODE} with the right hand side \eqref{eq:rhs} has the unique solution on interval $[-s,s]$, where $s = \frac{1}{\sqrt{d_n} \|\beta\|_2} \cdot \min\{\frac{r_K}{M}, \frac{1}{\mathcal L }\}$.
\end{thm}

Note that  \cref{thm:ODE_existence} only guarantees the existence/uniqueness of the initial value problem \eqref{eq:ODE} on a time interval $[-s, s]$, whereas, for our method, we require the existence/uniqueness on the time interval $[0,T]$. 
Note that we could accomplish this by constraining $\|\beta\|_2$ to be sufficiently small. However, this is too restrictive and we alternatively define the set     
\begin{equation}
\label{e:B0}
\mathcal B_0 := \{\beta\in \mathbb R^{d\times d_n}\colon \textrm{ for each $i \in [N]$, the solution to \eqref{eq:ODE} uniquely exists on $[0,T]$} \}. 
\end{equation}
Note that  $\mathcal B_0$ contains a ball around the origin (by \cref{thm:ODE_existence}) and is star-shaped with respect to the origin. Later, it will be useful (for compactness) to additionally assume that there exists a constant $b>0$ such that $\| \beta \|_2 \leq b$, so we define the subset $\mathcal B_1 = \mathcal B_1(b)$
\begin{equation}
\label{e:B1}
\mathcal B_1 := \{\beta\in \mathcal B_0 \colon \| \beta \|_2 \leq b \}. 
\end{equation}

We collect the assumptions on the data, dictionary functions, and parameters $\beta$, $Q$ in the following.
\begin{assum} 
\label{assumptions}
\ \ 
\begin{enumerate}
\item The $N$ samples of data $x_i \in \mathbb R^d$ lie in the compact set $K \subset \mathbb R^d$.
\item The dictionary $\Xi$ consists of fixed $d_n$ candidate functions $\xi_\ell\colon \mathbb R^d \to \mathbb R$, which are Lipschitz continuous on $K$ with Lipschitz constant $\mathcal{L}$.
\item For some fixed (large) $b>0$, we assume $\beta \in \mathcal B_1(b)$, defined in \eqref{e:B1}. 
\item For fixed $k \ll d$, $Q \in \mathcal O_k$, defined in \eqref{e:Ok}
\end{enumerate}
\end{assum}

\subsubsection*{Loss function}
To train the model and obtain the DDR embedding, we introduce the  loss function 
\begin{align}\label{eq:objective_regular}
J(\beta,Q) &=  \frac{1}{N}\sum_{i \in [N]} ~ \underbrace{\|h_i(T)-Q^\ast Qh_i(T)\|^2}_{\text{$i$-th sample residual error}} ~+ ~\mu \cdot \underbrace{R(h_i;\beta)}_{\text{regularization}}.
\end{align}
Here, $\beta \in \mathcal B_1$ and $Q \in \mathcal O_k$ are model parameters and $\mu$ is a model hyperparameter that gives a trade-off between the  two terms in the objective \eqref{eq:objective_regular}. 
The first term is seen to be the mean squared residual error; it encourages the solutions to the ODE \eqref{eq:ODE} at time $T$ to lie in a lower dimension subspace. 
The second term is a regularization term, which we will discuss next.

We introduce a regularization term in the objective \eqref{eq:objective_regular} since there are many flows $\Phi=\beta \Xi$ which give the same final-time hidden variables $\{h_i(T)\}_{i \in [N]}$. We would like to choose a regularization term so that the resulting vector field $x \mapsto \beta \Xi(x)$ has very regular, smooth trajectories. We choose the regularization function   
\begin{equation}\label{eq:regularization}
    R(h;\beta) = \frac{1}{N} \sum_{i=1}^N \int_0^T  \|\beta\Xi(h_i(t))\|^2  \ dt. 
\end{equation}
Since we can trivially rewrite 
$R(h;\beta) = \frac{1}{N} \sum_{i=1}^N \int_0^T  \| \dot h_i(t) \|^2  dt$, 
we can interpret $R(h;\beta)$ as the \emph{mean total kinetic energy} of the trajectories. 
We can also interpret the regularization $R(h;\beta)$ in terms of the Lagrangian perspective for the  Benamou-Brenier formulation of the Wasserstein metric (see \cref{subsec:OT}). Namely, the regularization term $R(h;\beta)$ is the \emph{action} for the vector field which advects the time-parameterized probability measure $\mu_t(x) = \frac{1}{N} \sum_{i=1}^N \delta(x-h_i(t))$, $t \in [0,T]$, 
$$
R(h;\beta) = \int_0^T \int_{\mathbb R^d} \|\beta\Xi(x(t))\|^2 ~d\mu_t(x) ~ dt. 
$$
Thus, as in the definition of the 2-Wasserstein distance, the regularization term penalizes the deviation of the trajectories from the constant-speed linear path between initial and final positions. Similar ideas were used in \cite{Regularity} where the speed up in training is emphasized resulting from better-conditioned ODEs.

\subsubsection*{Optimization formulation} 
To train the DDR model, we formulate the ODE-constrained optimization problem, 
\begin{subequations}
\label{eq:optimization}
{\small
\begin{align}
J^\star = \argmin_{\beta,Q} \ &  J(\beta, Q),  \quad J(\beta, Q):= \frac{1}{N}\sum_{i \in [N]}  \left(\|h_i(T)-Q^\ast Qh_i(T)\|^2 + \mu\int_0^T \| \beta\Xi(h_i(t))\|^2~dt \right) 
\label{eq:obj_fun}
\\ \textrm{s.t.} \ & 
\beta \in \mathcal B_1(b), \quad 
    Q \in \mathcal O_k \\
    &  \frac{d }{dt}h_i(t) =\beta\Xi(h_i(t)), \quad h_i(0) = x_i,  
\label{eq:consts_ODE} 
\end{align}}
\end{subequations}
We will show that this 
ODE-constrained optimization problem
is well-posed (\cref{subsec:existence}) and derive a gradient-based optimization method for solving it in \cref{subsec:adjoint}.

\begin{rem} \label{rem:BetaZero}
It is be useful to consider the problem when $\beta = 0$. In this case, the ODE \eqref{eq:consts_ODE} is trivial and $h_i(T) = h_i(0) = x_i$. We obtain
$$
J(\beta = 0, Q) = \frac{1}{N} \| X - Q^\ast  Q X\|_F^2,
$$
where 
$X = [ x_1 \mid \cdots \mid x_N] \in \mathbb R^{d\times N}$. Assuming the  singular value decomposition, 
$X = U \Sigma V^\ast $, where
the singular values are arranged in decreasing order, \ie, $\sigma_1 \geq \cdots \geq \sigma_d$, the  Eckart–Young theorem gives that for all $Q \in \mathcal O_k$,
$$
J(\beta = 0, Q) \geq \sum_{i = k+1}^d \sigma^2_i(X) \equiv J_0,
$$
with equality attained by $Q = U_k^\ast$ where $U_k \in \mathbb R^{d\times k}$ are the first $k$ columns of $U$. It follows that $J^\star \leq J_0$. We interpret this computation as follows. By allowing $\beta$ to vary, the method finds a nonlinear transformation $x_i \mapsto h_i(T)$ so that the PCA of the transformed data has a smaller objective value than the original PCA objective value. 
\end{rem}

\subsection{A DDR-based generative model}
\label{s:GenModel}
An interesting property of the DDR framework is that an (approximate) decoding can be obtained by the time-reversal of the learned dynamical system and thus, the model can be extended as a generative model \cite{VAE}.
More precisely, the encoder  $\mathcal{E} \colon \mathbb R^d \to \mathbb R^k$ can be written $ y =  \mathcal{E}(x) =  Q h(T)$, where $h(T)$ is the the solution to  \eqref{eq:ODE} at time $t=T$ with initial condition $x$. \emph{Assuming zero residual training error}, we have that $h(T) = Q^\ast y$. In this case, the decoder, $\mathcal D\colon \mathbb R^k \to \mathbb R^d$, is exactly obtained by solving the time-reversed dynamical system, 
\begin{subequations} 
\label{e:TimeReverseDS}
\begin{align}
\frac{d}{dt} h(t) = \beta \Xi(h(t)) \\
h(T) = Q^\ast y
\end{align}
\end{subequations}
backwards in time from $t=T$ to $t=0$ and setting $\mathcal{D}(y) = h(0)$. In the case of zero training error, we have $\mathcal{D} \circ \mathcal{E}(x_i) =x_i$ for all training data $x_i \in \mathbb R^d$. 
In general, we thus define the \emph{decoder} $\mathcal D\colon \mathbb R^k \to \mathbb R^d$ to be $\mathcal{D}(y) = h(0)$, where $h(t)$, $t\in[0,T]$ satisfies \eqref{e:TimeReverseDS} with final condition given by $h(T) = Q^\ast y$. 

Moreover, if the data has distribution $\rho_0$, the low-dimensional representation has distribution $\rho_T = (\mathcal{E})_\# \rho_0$. Thus, using a density estimate $\nu$ of $\rho_T$ (\eg, kernel density estimate), the decoder could be used to generate new data via $(\mathcal{D})_\# \nu$. Finally, if we assume that the  distribution $\rho_0$ is supported on a low-dimensional manifold embedded in $\mathbb R^d$, then $\mathcal D\colon \mathbb R^k \to \mathbb R^d$ is a parameterization of the manifold, so the DDR framework can be used in the context of manifold learning.

\subsection{Existence of a minimizer} \label{subsec:existence}
In this section we show that the constrained optimization problem in \eqref{eq:optimization} is well-defined in the setting of  \cref{assumptions}. We will employ the direct method in the calculus of variations.

The following Lemma shows that for $\beta \in \mathcal B_0$, if the kinetic energy of the solution is bounded then so is the solution. 
\begin{lem}\label{lem:bound_h}
Let $\beta \in \mathcal B_0$.  If the solution $h(t)$ to the Cauchy problem 
\begin{align*}
\dot h(t)& =\beta\Xi(h(t)) \\
 h(0) &= x. 
 \end{align*} 
satisfies $\int_0^T \| \dot h(\tau) \|^2 d \tau \leq C$ for some constant $C$, then  there exists an $H > 0$ such that
$$\| h(t) \| \leq H, \qquad \qquad  \forall t \in [0,T].
$$
\end{lem}
\begin{proof} We compute
$$
C\geq \int_0^T \| \dot h(\tau) \|^2 d \tau 
\geq \int_0^t \| \dot h(\tau) \|^2 d \tau 
\geq \left\| \int_0^t 
\dot h(\tau) d \tau \right\|^2 
= \| h(t)- h(0) \|^2 
$$
This implies that 
$  \| h(t) \| - \|  h(0) \|  \leq 
\| h(t) - h(0) \| \leq \sqrt{C}$ so that 
$\|h(t) \| \leq |K| + \sqrt{C} =: H$. 
\end{proof}

Recalling \cref{rem:BetaZero}, we need only consider $(\beta,Q)$ such that $J(\beta,Q) \leq J_0$. Observing that 
$$
J(\beta,Q) \geq 
\frac{\mu }{N}\sum_{i \in [N]}  
\int_0^T \| \dot h_i(\tau)\|^2~d\tau,
$$ 
\cref{lem:bound_h} shows that there exists an $H>0$ such that we need only consider  $(\beta,Q)$ such that the corresponding hidden solutions, $h_i(t)$ for $i\in [N]$, are bounded by $H$, \ie, $\|h_i(t)\| \leq H$. For this fixed $H>0$, we define $\mathcal B_2 = \mathcal B_2(H)$ by
\[\mathcal B_2  := \{\beta\in \mathcal B_1(b) \colon   \text{  the solutions $h_i(t), \ i \in [N]$ to \eqref{eq:ODE}  are bounded with } \|h_i(t)\| \leq H, \ t \in [0,T]  \}.
\]
We next prove the Lipschitz continuity of the solution $h$ to \eqref{eq:ODE} at time $t$ with respect to the  parameter $\beta$. 
\begin{lem}\label{lem:h}
For $\beta$ and $\tilde{\beta}$ in $\mathcal{B}_2$, denote the solutions to  \eqref{eq:ODE} as $h$ and $\tilde{h}$ respectively. Then we have
\begin{equation}\label{ineq:h}
    \|h(t) - \tilde{h}(t) \| \leq C\|\beta-\tilde{\beta}\|_2, \qquad \forall t \in [0,T]
\end{equation}
for some constant $C>0$.
\end{lem}
\begin{proof}
Since $h$ and $\tilde{h}$  are solution to the ODEs \eqref{eq:consts_ODE}, we have
\[
\frac{d}{dt} \left( h- \tilde{h} \right)(t) = \beta \Xi(h) - \tilde{\beta} \Xi(\tilde{h}),
\]
with the initial condition $h(0)-\tilde{h}(0) = 0. $
Then we estimate 
 \begin{align*}
    \|h(t)-\tilde{h}(t)\| &\leq \int_0^t\|\beta \Xi(h(\tau)) - \tilde{\beta} \Xi(\tilde{h}(\tau))\|  d\tau \\
%    & =  \int_0^t\|\beta \Xi(h(\tau)) -\tilde{\beta} \Xi(h(\tau)) + \tilde{\beta} \Xi(h(\tau))-\tilde{\beta} \Xi(\tilde{h}(\tau))\|  d\tau \\
  &\leq \int_0^t \|\beta-\tilde{\beta}\|_2 \|\Xi(h(\tau)\| ~d\tau + \int_0^t \|\tilde{\beta}\|_2 \|\Xi(h(\tau)) -\Xi(\tilde{h}(\tau))\|~d\tau \\
    &\leq \int_0^t \|\beta-\tilde{\beta}\|_2 \|\Xi(h(\tau)\| ~d\tau + \int_0^t \sqrt{d_n}\mathcal{L}M\|h(\tau) -\tilde{h}(\tau)\|~d\tau.
\end{align*}
By Gronwall's inequality, $
 \|h(t)-\tilde{h}(t)\| \leq \|\beta-\tilde{\beta}\|_2  \int_0^t \|\Xi(h(\tau))\|~d\tau e^{\sqrt{d_n}\mathcal{L}Mt}$. 
Since $\beta\in \mathcal{B}_2$, $h(t)\leq H$, which implies that the norm of the finite number of dictionary terms at $h(t)$  is also bounded, \ie,~ $\|\Xi(h(t))\| \leq \tilde{H}$, for some constant $\tilde{H}>0$ over interval $t\in [0,T]$. Therefore, \eqref{ineq:h} holds with $C=T\tilde{H} e^{\sqrt{d_n}\mathcal{L}M T}$. 
\end{proof}

Using \cref{lem:h}, we prove the following theorems establishing continuity and compactness of the objective function over  $\mathcal{B}_2\times \mathcal{O}_k$. 

\begin{thm}\label{thm:cont}
The objective function  $J$ is $2-$Hölder continuous over $\mathcal{B}_2$ and $\mathcal{O}_k$ respectively, \ie ~ for all $\beta, \tilde{\beta}\in \mathcal{B}_2$ and $Q,P \in \mathcal{O}_k$,
\begin{equation}\label{ineq:cont_Q}
    |J(\beta,Q)-J(\tilde{\beta},P)| \leq C_1 \|Q-P\|^2 + C_2 \|\beta-\tilde{\beta}\|^2,
\end{equation}
for some positive constants $C_1$ and $C_2$.
\end{thm}
\begin{proof}
For any $Q,P\in \mathcal{O}_k$, 
\begin{equation}\label{ineq:QP}
    \|Q^\ast Q -P^\ast P\| \leq \|Q^\ast Q-Q^\ast P\| + \|Q^\ast P-P^\ast P\| \leq (\|Q\|+\|P\|)\|Q-P\| = 2\|Q-P\|, 
\end{equation}
where orthonormality of $Q$ implies $\|Q\| = \sqrt{\lambda_{max}(QQ^\ast )} = 1$.  Then for fixed $\beta\in \mathcal{B}_2,$
{\small
\begin{align*}
    |J(\beta,Q)-J(\beta,P)| &= \frac{1}{N}\sum_{i \in [N]}\left|\|(I-Q^\ast Q)h_i(T)\|^2 - \|(I-P^\ast P)h_i(T)\|^2
    \right| \\
    & \leq \frac{1}{N}\sum_{i \in [N]}\|(I-Q^\ast Q)h_i(T) - (I-P^\ast P)h_i(T)\|^2 \leq \frac{1}{N}\sum_{i \in [N]}\|h_i(T)\|^2 \|Q^\ast Q-P^\ast P\|^2\\
    & \leq C_1\|Q-P\|^2, 
\end{align*}}
where $C_1 = 2H^2$ and  the last inequality is obtained by \eqref{ineq:QP}. Therefore, the objective is Hölder continuous w.r.t $Q$.
Moreover, for $\beta, \tilde{\beta} \in \mathcal{B}_2,$
{\small
\begin{align*}
    |J(\beta,P)-J(\tilde{\beta},P)| & \leq \frac{1}{N}\sum_{i\in [N]}\left| \|(I-Q^\ast Q)h_i(T)\|^2 - \|(I-Q^\ast Q)\tilde{h}_i(T)\|^2\right|  + \mu \int_0^T \left| \|\beta \Xi(h_i)\|^2 - \|\tilde{\beta} \Xi(\tilde{h_i})\|^2  \right|~d\tau\\
    &\leq \frac{1}{N}\sum_{i\in [N]}\|(I-Q^\ast Q)(h_i(T) -\tilde{h}_i(T))\|^2 + \mu \int_0^T \|\beta \Xi(h_i)-\tilde{\beta} \Xi(\tilde{h_i})\|^2  ~d\tau\\
     & \leq \frac{1}{N}\sum_{i\in [N]} \|I-Q^\ast Q\|^2\|h_i -\tilde{h}_i\|^2 + \mu \int_0^T \|\beta-\tilde{\beta}\|^2\|\Xi(h_i(\tau)\|^2 + \|\tilde{\beta}\|^2\|\Xi(h_i)-\Xi(\tilde{h}_i)\|^2~d\tau\\
    &\leq \frac{1}{N}\sum_{i\in [N]} \|h_i -\tilde{h}_i\|^2 +\mu \int_0^T \|\beta-\tilde{\beta}\|^2\tilde{H}^2 + d_n\mathcal{L}^2M^2\|h_i-\tilde{h}_i\|^2~d\tau\\
    &\leq (1+\mu T d_n\mathcal{L}^2M^2)\|h-\tilde{h}\|^2 + \mu T\tilde{H}^2 \|\beta-\tilde{\beta}\|^2.
\end{align*}}
Using the inequality in \cref{lem:h}, 
\[
 |J(\beta,P)-J(\tilde{\beta},P)| \leq  C_2\|\beta-\tilde{\beta}\|^2,
\]
where  $C_2 =C^2 + \mu T( \tilde{H^2} + d_n \mathcal{L}^2C^2 M^2)$. Therefore, we have 
$$
|J(\beta,Q)-J(\tilde{\beta},P)| \leq 
|J(\beta,Q)-J(\beta,P)| + 
|J(\beta,P)-J(\tilde{\beta},P)| \leq C_1 \|Q-P\|^2 + C_2 \|\beta-\tilde{\beta}\|^2.
$$
\end{proof}

\begin{thm}\label{thm:compact}
The feasible set $\mathcal{B}_2 \times \mathcal{O}_k$ is compact. 
\end{thm} 
\begin{proof}
Since $Q\in \mathcal O_k \subset \mathbb R^{k \times d}$ has orthonormal rows, $\|Q\|_2 \leq 1$ and $\mathcal O_k$ is bounded. To show  $\mathcal O_k$ is closed,  define a mapping $f\colon \mathbb R^{k \times d} \mapsto \mathbb R^{k\times k}$ such that $f\colon A \mapsto AA^\ast $. Arguing as in \eqref{ineq:QP}, for $Q,P \in \mathcal O_k$, we have 
\begin{equation*}
   \|f(Q)-f(P)\| =  \|QQ^\ast -PP^\ast \| \leq 2\|Q-P\|,
\end{equation*}
so $f$ is continuous. Since $\mathcal O_k = f^{-1}(\{I_k\})$ and the singleton is closed,  $\mathcal O_k $ is also closed. 

Following the definition of $\mathcal{B}_2$, it is bounded. Suppose a sequence $\{\beta_j\}_{j \in \mathbb{N}}\subset\mathcal{B}_2$ converges to $\tilde{\beta}$. Then we can define a sequence  $\{h_j\}_{j \in \mathbb{N}}$ of solution to \eqref{eq:ODE} corresponding to  $\{\beta_j\}_{j \in \mathbb{N}}$, which is equivalent to   $h_j\colon [0,T] \mapsto \mathbb R^d$ solves the integral equation
\begin{equation}\label{eq:integral}
    h_j(t) = x + \int_0^t \beta_j \Xi(h_j(\tau))~d \tau,\qquad \forall t\in [0,T].
\end{equation}
By the definition of $\beta_j \in \mathcal{B}_2$,  a sequence $\{h_j\}_{j \in \mathbb{N}}$ is uniformly bounded, that is, 
\[
\|h_j(t)\| \leq H, \qquad \forall j\in \mathbb{N},  t\in [0,T].
\]
Moreover, for arbitrary $ t,s \in [0,T]$, we have
\[
\|h_j(t)-h_j(s)\| \leq \int_s^t \|\beta_j \Xi(h_j(\tau))\|~d\tau 
\leq \int_s^t \|\beta_j\|_2 \|\Xi(h_j(\tau))\| ~d\tau 
\leq M \tilde{H}|t-s|, \quad \forall j \in \mathbb{N}.
\]
Hence a sequence $\{h_j\}_{j \in \mathbb{N}}$ is uniformly equicontinuous. By the Arzela-Ascoli theorem, there exists subsequence $\{h_j\}_{j \in \mathbb{N}}$, denoted with same index, such that it converges uniformly, say $h_j \to \tilde{h}$. By the continuity of dictionary $\Xi$,
\[
\tilde{h}(t)= \lim_{j\rightarrow \infty}h_j(t)=  x +  \int_0^t \left(\lim_{j\rightarrow \infty}\beta_j\right)\left(\lim_{j\rightarrow \infty}\Xi(h_j(\tau))\right)~d\tau = x + \int_0^t \tilde{\beta} \Xi(\tilde{h}(\tau))~d\tau.
\]
Hence, the ODE \eqref{eq:ODE} with $\tilde{\beta}$ is uniquely solved in $[0,T]$. Also, by the continuity of the norm, 
\[
\|\tilde{\beta}\| = \|\lim_{j\rightarrow \infty}\beta_j\| = \lim_{j\rightarrow \infty}\|\beta_j\| \leq M,\quad \text{and}\quad \|\tilde{h}(t)\| = \|\lim_{j\rightarrow \infty}h_j(t)\| = \lim_{j\rightarrow \infty}\|h_j(t)\| \leq H.
\]
Therefore, $\tilde{\beta} \in \mathcal{B}_2$ and thus $\mathcal{B}_2$ is compact. 
\end{proof}

Finally, we use \cref{thm:cont} and \cref{thm:compact} to prove the following result that the constrained minimization problem \eqref{eq:optimization} is well-defined.

\begin{thm}\label{thm:ex_opt}
There exists a $( \beta_\star, Q_\star)\in \mathcal B_2 \times \mathcal O_k$ that attains the infimum value 
$$
\inf_{(\beta, Q) \in \mathcal B_1 \times \mathcal O_k} J(\beta, Q),
$$
where the ODE constraints \eqref{eq:consts_ODE}  are implicit.
\end{thm}
\begin{proof}
We argue via the direct method in the Calculus of Variations. We know $J(\beta, Q) \geq 0$ for all $(\beta, Q) \in \mathcal B_1 \times \mathcal O_k$. 
We take a minimizing sequence $(\beta_j, Q_j) \subset \mathcal B_1 \times \mathcal O_k$. Since this is a minimizing sequence, by \cref{lem:bound_h}, we know that there exists a constant $H>0$, such that 
$(\beta_j, Q_j) \subset \mathcal B_2(H) \times \mathcal O_k$. 
By compactness (\cref{thm:compact}), we can extract a convergent subsequence, which we again index $(\beta_j, Q_j)$, such that 
$\lim_{j\to \infty} (\beta_j, Q_j) = (\beta_\star , Q_\star)$.
Now using the continuity of $J$ (\cref{thm:cont}), we have 
$$
J^\star = \inf_{(\beta, Q) \in \mathcal B_2 \times \mathcal O_k} J(\beta, Q) 
= \lim_{j \to \infty} J (\beta_j, Q_j) = J( \beta_\star, Q_\star).
$$
So, $( \beta_\star, Q_\star)\in \mathcal B_2 \times \mathcal O_k$
attains the infimum value.
\end{proof}

Although \cref{thm:ex_opt} gives the existence of a solution, we do not necessarily have a unique solution. Of course, this is also the case for PCA if the singular values have a multiplicity greater than one. 

We also remark that our method is not identifiable. We illustrate this in \cref{subsec:S}, where we train our model for a synthetic dataset which is formed using a known vector field and orthonormal subspace. We find that the learned parameters can differ from the ground truth.  

\subsection{Gradient computations}\label{subsec:adjoint}
We use a gradient-based optimization method to solve \eqref{eq:optimization} and learn the parameters for the DDR model. To compute the gradient of the loss with respect to each parameter, we apply the adjoint method. 

\begin{thm}\label{thm:adjoint}
The gradients of the objective function \eqref{eq:obj_fun} with respect to the parameters $\beta \in \mathcal{B}$ and $Q \in \mathcal{Q}_k$ are given by 
\begin{subequations}\label{eq:gradients}
 \begin{align}
d_\beta J &  = \frac{1}{N}\sum_{i \in [N]} \int_0^T  \left[  \frac{2\mu}{T}\beta \Xi(h_i) \Xi(h_i)^\ast - \lambda_i\Xi(h_i)^\ast \right] dt  \label{eq:grad_beta}\\
d_Q J & =\frac{1}{N}\sum_{i  \in[N]
}2Qh_i(T)h_i(T)^\ast Q^\ast Q- 2Qh_i(T)h_i(T)^\ast  \label{eq:grad_Q},
\end{align}
\end{subequations}
where $\lambda_i(t): [0,T] \to \mathbb R^d$ is a solution to the adjoint equation, for all $i\in[N]$
\begin{subequations}\label{eq:adjoint}
 \begin{align}
\frac{d\lambda_i}{dt} &=  -\nabla\Xi(h_i)^\ast  \beta^\ast \lambda_i + \frac{2\mu}{T}\nabla\Xi(h_i)^\ast  \beta^\ast  \beta \Xi(h_i) \label{eq:adjoint_ODE}\\
\lambda_i(T) & =-2 (I_d-Q^\ast Q)h_i(T)\label{eq:adjoint_T}.
\end{align}
\end{subequations}
\end{thm}
\begin{proof}
Let the Lagrangian multipliers $\lambda_i(t)[0,T] \to \mathbb R^d$ be given. Then  the Lagrangian is defined as 
\begin{align*}
      \mathcal L (\lambda_i) &=\frac{1}{N}\sum_{i\in[N]} \int_0^T \left(\frac{\mu}{T}|\beta \Xi(h_i)|^2 + \lambda_i^\ast  \dot{h_i}-\lambda_i^\ast  \beta\Xi(h_i)\right)~dt + \|h_i(T)-Q^\ast Qh_i(T)\|^2.
\end{align*}
Using the integration by parts, 
$\int_0^T  \lambda_i^\ast  \dot{h_i} dt = [\lambda_i^\ast h_i]_0^T - \int_0^T \dot{\lambda_i^\ast }h_i dt$, 
the Lagrangian can be rewritten as 
\begin{align*}
      \mathcal L (\lambda_i) &=\frac{1}{N}\sum_{i\in[N]} \int_0^T \left(\frac{\mu}{T}|\beta \Xi(h_i)|^2-\dot{ \lambda_i}^\ast  h_i-\lambda_i^\ast  \beta\Xi(h_i)\right)dt \\
      & +\|h_i(T)-Q^\ast Qh_i(T)\|^2  +\lambda_i(T)^\ast h_i(T) - \lambda_i(0)^\ast h_i(0).
\end{align*}
Taking the total derivative of $
\mathcal L$ w.r.t $\beta,Q$,   we obtain
\begin{align*}
      d_\beta \mathcal L &=\frac{1}{N}\sum_{i\in[N]} \int_0^T \left[\left(\frac{2\mu}{T}\Xi^\ast \beta^\ast  \beta \Xi -\lambda_i^\ast \beta \nabla \Xi  -\dot{\lambda_i}^\ast  \right)d_\beta h_i + \frac{2\mu}{T} \beta \Xi \Xi^\ast  - \lambda_i^\ast  \Xi\right]~dt  \\
      & ~+  d_\beta h_i(T)(\lambda_i(T)^\ast   + 2(I_d-Q^\ast Q )h_i(T)),\\
        d_Q \mathcal L &=\frac{1}{N}\sum_{i\in[N]} \int_0^T \left(\frac{2\mu}{T}\Xi^\ast \beta^\ast  \beta \Xi -\lambda_i^\ast \beta \nabla \Xi  -\dot{\lambda_i}^\ast  \right)d_Q h_i ~dt   \\
        &~+ 2Q(Q^\ast Qh_i(T)h_i(T)^\ast  +h_i(T)h_i(T)^\ast Q^\ast Q- 2h_i(T)h_i(T)^\ast ) \\
      & ~+  d_Q h_i(T)(\lambda_i(T)^\ast   + 2(I_d-Q^\ast Q)h_i(T)).
\end{align*}
Since $d_\beta h_i $ and $d_Q h_i$ are expensive to compute, we solve the adjoint equation alternatively. By setting $d_\beta \mathcal L$ and $d_Q \mathcal L$ to be zero, we derive the adjoint equations \eqref{eq:adjoint} and 
the gradients of the objective with respect to $\beta$ and $Q$ are then formulated as \eqref{eq:gradients}. 
\end{proof}

\subsection{Solution to the \texorpdfstring{$Q$}{Q}-subproblem.} \label{s:Qsubproblem}
Let  $H_T = [h_1(T)\mid \cdots \mid h_N(T)] \in \mathbb R^{d\times N} $ be a matrix of hidden variables  $h_i(T)$ at the final time $T.$  To seek the optimal $Q^\ast$ of the problem \eqref{eq:optimization}, we consider the  \textit{$Q$-subproblem} minimizing a mean squared residual error 
\begin{equation}\label{eq:Q_sub}
    Q^\star = \argmin_{Q\in \mathcal O_k}\frac{1}{N}\sum_{i \in [N]}  \|h_i(T)-Q^\ast Qh_i(T)\|^2  =\argmin_{Q\in \mathcal O_k} \frac{1}{N}\|(I-Q^\ast Q)H_T\|^2_F.
\end{equation}
The following lemma provides the solution to the $Q$-subproblem.  

\begin{lem}
For $Q \in \mathbb R^{k\times d}$ with orthonormal rows and $H_T \in \mathbb R^{d \times N}$ with SVD $H_T = U \Sigma V^\ast $, we have that 
$$
\| (I-Q^\ast Q) H_T \|_F^2 \geq \sum_{i=k+1}^d \sigma_i^2(H_T)
$$
with equality attained by $Q = U_k^\ast $, where $U_k$ are the first $k$ columns of $U$, corresponding to the largest singular values, $\{ \sigma_i\}_{i=1}^k$. Thus, the solution to the $Q$-subproblem is explicitly given by $Q^\star = U_k^\ast $. 
\end{lem}
\begin{proof}
First note that since $Q^\ast Q$ is a projection matrix, 
\begin{align*}
\| (I-Q^\ast  Q) H_T \|_F^2 &= \| H_T \|_F^2 - \| Q^\ast  Q  H_T \|_F^2 \\
&= \sum_{i=1}^d \sigma_i^2(H_T) - \langle Q^\ast  Q, H_T H_T^\ast  \rangle
\end{align*}
Fan's inequality states that for any symmetric matrices $X$ and $Y$, we have that 
$\langle X,Y \rangle \leq \lambda(X) \lambda(Y)$,
where $\lambda$ denotes the eigenvalues listed in non-increasing order. Furthermore, equality holds if and only if X and Y have a simultaneous ordered spectral decomposition \cite{Borwein_2000}. Since the eigenvalues of the projection matrix $Q^\ast  Q$ are 
$\lambda = 1$ with multiplicity $k$ and
$\lambda = 0$ with multiplicity $d-k$,
we have that 
$$
\langle Q^\ast  Q, H_T H_T^\ast  \rangle 
\leq \sum_{i=1}^k \lambda_i(H_T H_T^\ast  )
= \sum_{i=1}^k \sigma^2_i(H_T)
$$
with equality if and only if there exists an orthogonal $\tilde U$ such that 
$Q^\ast  Q = \tilde U \textrm{diag} \lambda(Q^\ast  Q) \tilde U^\ast $
and $H_T H_T^\ast  = \tilde U \textrm{diag} \lambda(H_T H_T^\ast ) \tilde U^\ast $. Clearly, we can pick $\tilde U = U$ and $Q^\ast  Q = U_k U_k^\ast $. 
\end{proof}

\subsection{An algorithm for the solution of the DDR model} \label{s:algorithm}
There are several different approaches to solve the optimization problem for DDR model \eqref{eq:optimization}. One approach would be a projected gradient-based method. Here, the gradients in \cref{thm:adjoint} would be use to take a gradient-based step (\eg, a stochastic gradient descent step) and then the updated $Q$ would be projected onto the constraint set $\mathcal O_k$. Instead, we use an alternating method, summarized in \cref{alg:DDR}, which uses the exact solution for the $Q$-subproblem (see \cref{s:Qsubproblem}).

\begin{algorithm}[t!]
\caption{Dynamical Dimension Reduction}\label{alg:DDR}
\begin{algorithmic}[]
\State {\bf Input:}  initial parameters, $\beta, Q$.
\For{epoch = 1, \ldots, $N_{\mathrm epoch}$:}
\State{Shuffle data and create batches of size $N_{batch}$}
\For{each batch:}

\State {\bf (Solve the forward ODE for $h_i$)} For the current parameters $\beta$, solve the forward\newline\hspace*{3em}
ODE \eqref{eq:ODE},  \ie, for each example $i \in [N_{batch}]$ and discrete times $t_m$, $m\in [T_m]$,  find $h_i(t_m)$.\\

 \State{\bf (Solve $Q$-subproblem)} Using the hidden state at the final time $h_i(T)$, solve 
 \newline\hspace*{3em} 
 the \textit{$Q$-subproblem} in \eqref{eq:Q_sub} \ie ~$H_T = U\Sigma V^\ast$ and update $Q\leftarrow U_k^\ast $.\\ 

\State{\bf (Solve the adjoint equation for  $\lambda_i$)} Using updated $Q$ and $h_i(t_m) $, compute 
 \newline\hspace*{3em} the terminal condition and solve the backward ODE in  \eqref{eq:adjoint}\newline\hspace*{3em}  \ie, 
for each example $i \in [N_{batch}]$ and discrete times $t_m$, $m\in [T_m]$, find $\lambda_i(t_m)$.\\

\State {\bf (Compute gradients)} Using $h_i(t_m) $ and $\lambda_i(t_m)$, evaluate the gradient of the objective\newline\hspace*{3em}  function with respect to the parameters $\nabla_\beta J$  as in \eqref{eq:gradients}.\\

\State{\bf (Update $\beta$)} Use a gradient-based optimization method, \eg, gradient descent or
\newline\hspace*{3em} ADAM method, to update the parameters, $\beta$. \\ 

\EndFor
\EndFor
\end{algorithmic}
\end{algorithm}

\newpage 

\section{Properties of the Dynamical Dimension Reduction Model}\label{sec4}
Here we present a few properties of the DDR model. In \cref{subsec:stability} we describe stability/generalizability of the forward model. In \cref{subsec:PCA} we describe the reduction to PCA for a linear dictionary.

\subsection{Stability/generalizability of the forward model}
\label{subsec:stability}
In this section, we prove that the dynamical dimension embedding $ x_i \xmapsto{\mathcal{E}} y_i$ is stable under the perturbation in given data. Denote the optimal parameter of \eqref{eq:optimization} by
$(\beta^\star , Q^\star )$  and $h_i(T)$ the solution to \eqref{eq:ODE} with parameters $\beta^\star $, so that
$y_i = \mathcal{E}(x_i) = Q^\star h_i(T)$. 
\begin{thm}\label{thm:stability}
Consider the dimension reduction embedding $\mathcal{E}\colon \mathbb R^d \to \mathbb R^k$ with dictionary $ \Xi$ satisfying   \cref{assumptions}. Then mapping $\mathcal{E}$ is Lipshcitz continuous, \ie ~ for $x_1,x_2 \in \mathbb R^d$,  
\[
\|\mathcal{E}(x_1)-\mathcal{E}(x_2) \|  \leq C \|x_1-x_2\| ,
\]
$C>0$ is a constant described in the proof.
\end{thm} 
\begin{proof}
Consider ODEs of hidden variable $h_1$ and $h_2$ with unperturbed and perturbed initial condition respectively,
\begin{align*}
\left\{
\begin{matrix}
\dot{h_1} = \beta\Xi(h_1)\\h_1(0) = x_1
\end{matrix}\right. \quad \text{and}\quad \left \{
\begin{matrix}
\dot{h_2}  = \beta\Xi(h_2)\\
h_2(0) = x_2.
\end{matrix}
\right.
\end{align*}
By subtracting these equations, we estimate
\[
\|h_1(t)-h_2(t)\| 
\leq \|x_1-x_2\|  + \int_0^t \|\beta\|\|\Xi(h_1(\tau))-\Xi(h_2(\tau))\|~d\tau
\leq \|x_1-x_2\|  + \int_0^t M\mathcal{L}\sqrt{d_n}\|h_1(\tau)-h_2(\tau)\|~d\tau.
\]
Gronwall's inequality yields 
\[
\|h_1(t)-h_2(t)\| 
\leq \|x_1 - x_2 \|e^{tM\mathcal{L}\sqrt{d_n}}
\leq \tilde{C} \|x_1 - x_2 \|,
\]
where $\tilde{C} = e^{TM\mathcal{L}\sqrt{d_n}}$ is a constant. Then  low dimensional representation of each data from the mapping provides 
\[
\|\mathcal{E}(x_1)-\mathcal{E}(x_2)\| = \|Q h_1(T) - Q h_2(T)\| \leq \tilde{C} \|Q\|\| x_1 - x_2  \|,
\]
as desired.
\end{proof}
The  \cref{thm:stability} can be interpreted as the generalizability of our model. Suppose new data $x_\eta$ is in the $\eta$-ball of the original data $x$ used for the training mapping. Then output  $\mathcal{E}(x_\eta)$  of embedding doesn't move far from $\mathcal{E}(x)$. Thus we could obtain a reliable lower dimensional representation of new data without retraining the model. Later, we will  illustrate the stability of our embedding model under the noise in a given data through numerical experiments in \cref{subsec:S}.

\begin{thm}
Suppose the residual training error is zero. Suppose $x_1, x_2$ are two points in the training dataset and $y_1 = \mathcal{E}(x_1)$ and $y_2 =\mathcal{E}(x_2)$ are the embedded points. Then 
$$
\| y_1 - y_2 \| \geq C^{-1} \| x_1 - x_2\|,
$$
where $C>0$ is the same constant as in \cref{thm:stability}. In particular, this implies that the embedding $\mathcal{E}\colon  \mathbb R^d \to \mathbb R^k$ is a quasi-isometry on the training data, \ie, it satisfies 
$$
C^{-1 }\|x_1 - x_2 \|  
\leq  \| \mathcal{E}(x_1) - \mathcal{E}(x_2) \| 
\leq C \|x_1 - x_2 \|. 
$$
\end{thm}
\begin{proof}
Since there is zero training error, we have 
$h_i(T) = Q^\ast Q h_i(T)$. 
So, 
$$
\| y_1 - y_2 \| 
= \| Q h_1(T) - Q h_2(T) \| 
= \| Q^\ast Q h_1(T) - Q^\ast Q h_2(T) \| 
= \| h_1(T) - h_2(T) \|.
$$
But now using Gronwall's inequality for the time-reversed dynamical system, we obtain
$$
\| x_1 - x_2 \| 
\leq C \| h_1(T) - h_2(T) \| 
= C \| y_1 - y_2 \|, 
$$
which proves the first claim. 

The second claim now follows from \cref{thm:stability}.
\end{proof}

The next theorem gives a result of for the decoder, discussed in \cref{s:GenModel}. 

\begin{thm} \label{t:Decoder}
Suppose the residual training error is zero and let $\{ y_i \}_{i \in [N]} \subset \mathbb R^k$ be the DDR embedded points. There exists an open set $Y \supset \{ y_i \}_{i \in [N]}$ such that for every $y \in Y$, $\mathcal D(y)$ is finite.
\end{thm}
\begin{proof}
Fix $i \in [N]$. There exists a neighborhood $Y_i \ni y_i$ such that for every $y \in Y_i$ there exists a unique solution to \eqref{e:TimeReverseDS} at time $t=0$ with final  condition given by $h(T) = Q^\ast y$. We simply take $Y = \cup_{i\in [N]} Y_i$. 
\end{proof}

\subsection{Connections with principal component analysis }
\label{subsec:PCA}
In this section, we consider again the motivating PCA-based example described in \cref{s:IntroExample}. 
Suppose we have data $X = [ x_1 \mid \cdots \mid x_N] \in \mathbb R^{d\times N}$ with $N > d$ with singular value decomposition, 
$X = U \Sigma V^\ast$, where
the singular values are arranged in decreasing order, \ie, $\sigma_1 \geq \cdots \geq \sigma_d$. 
We consider the application of the DDR method with a \emph{linear dictionary}. \ie~ $d_n = d$ and 
$\Xi(h) = [h_1, \ldots, h_d]$. 
We consider 
\begin{equation} \label{e:PropSol}
\beta =A_\varepsilon= \frac{1}{T}U \text{diag}[0,\dots,0,\log\varepsilon,\dots, \log\varepsilon ]U^\ast  
\quad \textrm{and} \quad Q= U_k^\ast . 
\end{equation}
where $\varepsilon>0$ and $U_k \in \mathbb R^{d\times k}$ are the first $k$ columns of $U$.  The following lemma shows that the proposed solution  \eqref{e:PropSol} is a stationary point for \eqref{eq:optimization} for a particular choice of $\varepsilon$.

\begin{lem}\label{lem:opt_linear}
For fixed $\mu >0$, we consider the DDR method with a linear dictionary with the notation introduced above. We consider $(\beta,Q) = (A_\varepsilon,U_k^\ast )$ given in \eqref{e:PropSol}. 
For any $\varepsilon >0$,  $d_Q J(A_\varepsilon,U_k^\ast ) = 0$.
For $\mu>0$, there exists a unique $\varepsilon^\star = \varepsilon^\star(\mu) \in [ e^{-1}, 1)$, such that 
$d_\beta J(A_{\varepsilon^\star},U_k^\ast )  = 0$.
In particular, $(\beta,Q) = (A_{\varepsilon^\star},U_k^\ast )$ is a stationary point for \eqref{eq:optimization}.
\end{lem}

\begin{proof}
For this choice of $(\beta,Q)$, the solution to ODE in \eqref{eq:consts_ODE}  is simply 
$$
h_i(t) =U D(1,\varepsilon^\frac{t}{T}) U^\ast x_i.
$$
where we use the notaiton $D(a,b) = \text{diag}(a,\ldots,a,b,\ldots,b)$ where the $a$ is repeated $k$ times and $b$ is repeated $d-k$ times. 
We evaluate $d_Q J$ in \eqref{eq:grad_Q} at $Q = U_k^\ast$ to obtain 
$d_Q J = 0$.

The adjoint equation \eqref{eq:adjoint} is then written 
\begin{align*}
\frac{d\lambda_i}{dt} &=  -\frac{1}{T} U D( 0,\log\varepsilon)
U^\ast \lambda_i + \frac{2\mu}{T^3}  U  D\left(0,     \varepsilon^\frac{t}{T} \log^2 \varepsilon \right) U^\ast x_i \\
\lambda_i(T) & =-2 \varepsilon U D(0,1) U^\ast x_i.
\end{align*}
The solution to the adjoint equation, which can be derived using variation of parameters, is given by
$$
\lambda_i(t) = -\left[ 2 \varepsilon^{2-t/T}+  \frac{\mu \varepsilon \log \varepsilon}{T^2}\left(\varepsilon^{1-t/T }- \varepsilon^{-(1-t/T)}\right)\right]U D\left(0, 1 \right)U^\ast x_i. 
$$
We now compute the derivative of the objective function with respect to $\beta$ using \eqref{eq:grad_beta} and the explict solutions for $h_i(t)$ and $\lambda_i(t)$ derived above. We obtain
$$
d_\beta J = f(\varepsilon, \mu) U_{-k} D(0, 1) \Sigma U_{-k}^\ast,
 \qquad \textrm{where } \ 
 f(\varepsilon, \mu):= \frac{T}{N}\left[-\frac{\mu}{2T^2} -\left(-\frac{2\varepsilon^2 \log \varepsilon}{T^2} \right) + \varepsilon^2 \left( 2+\frac{\mu}{2T^2} \right)
\right].
$$
This gives 
$$
\| d_\beta J \|_F^2 = f^2(\varepsilon, \mu) \sum_{i=k+1}^d
 \sigma_i^2.
 $$
We claim that for fixed $\mu >0$ there exists a unique $\varepsilon_\star = \varepsilon_\star(\mu) >0$ such that $f(\mu,\varepsilon_\star(\mu)) = 0$. 
First, we solve the equation for $\mu$
 $$
 \mu(\varepsilon) = \frac{4\varepsilon^2\log \varepsilon + 4\varepsilon^2 T^2}{1-\varepsilon^2}
 $$
This function of $\varepsilon$ is a mapping from $[e^{-1},1)$ onto $[0,\infty)$ such that for $T\geq 1$ it is monotonically increasing with strictly positive derivative on $[e^{-1},1)$.  By the inverse function theorem, $\mu(\varepsilon)$ is invertible and we attain a unique $\varepsilon^\star(\mu)$ for any $\mu\geq 0$. 
\end{proof}

We can approximate the function $\varepsilon^\star(\mu)$ 
guaranteed to exist in \cref{lem:opt_linear} by expanding $\mu$ in a power series about $\varepsilon = e^{-1}$,
\[
\mu(\varepsilon) =  c_1\left(\varepsilon-e^{-1}\right) + c_2 \left(\varepsilon-e^{-1}\right)^2
 + \mathcal{O}\left(\left(\varepsilon-e^{-1}\right)^3 \right),\]
 where $c_1 = \frac{4e}{e^2-1}$ and $c_2 = \frac{2(e^2+3e^4)}{(e^2-1)^2}.$ 
Solving for $\epsilon$, we obtain the approximation 
\begin{equation}\label{eq:app_eps}
\varepsilon^\star(\mu) \approx e^{-1} + \frac{-c_1 + \sqrt{c_1^2 +4c_2\mu}}{4c_2}.
\end{equation}
We will use the approximate value \eqref{eq:app_eps} in the initialization of our method.

\section{Model implementation and numerical experiments}\label{sec:experiments}
In this section, we describe an implementation of the Dynamical Dimension Reduction (DDR) method and describe its performance on a variety of datasets, including 
an S-shaped synthetically generated dataset  (\cref{subsec:S}), 
the iris dataset (\cref{subsec:iris}), and 
handwritten digit dataset (\cref{subsec:digit}). 
We demonstrate that the DDR method attains a desirable dimension reduction and compares with embeddings generated by PCA, t-SNE and Umap. 
Source code for our implementation is available at  \url{https://github.com/rkyoon12/DDR}.

\subsection{Model implementation}
We implemented the DDR method summarized in \cref{alg:DDR} and described in \cref{subsec:main}. The optimization problem \eqref{eq:optimization} was solved using the ADAM gradient-based method, implemented via the JAX library  \verb+jax.example_libraries.optimizers.adam+ with the learning rates decaying from $lr = 0.01$ to $lr = 0.001$ as the iterations progress. The gradients of the objective function with respect to the parameters are computed by the formula in \eqref{eq:gradients}. The solutions to both the forward ODE  \eqref{eq:consts_ODE} of the hidden variable $h\colon [0,T]\to \mathbb{R}^d$  and the adjoint equation \eqref{eq:adjoint} of the Lagrange multiplier $\lambda\colon [0,T]\to \mathbb{R}^d$ are used to compute the gradients. We use the given data as the initial condition in \eqref{eq:consts_ODE} and set the terminal time $T=1$.  The forward Euler ODE solver is used with discretized interval with time step  $dt = 0.01$. 
To avoid blow-up for the solutions to ODEs in the given time interval and ensure the convergence of algorithms, we threshold the values of the state variable $h$ and adjoint variable $\lambda$ pointwise to be less than 100.  

\subsubsection*{Dictionary choice and initialization}
 Since we employ the solution of nonlinear ODEs to define the objective function, the initialization for $\beta$ is important in the convergence of our model. As shown in \cref{thm:ODE_existence}, we initialize $\beta$ to satisfy the assumption in \eqref{e:B1} for pre-specified dictionaries. 
 There is a lot of freedom in choosing the dictionary elements; one could use, \eg, polynomials, multinomials, trigonometric functions, {\it etc}. 
 For example, we choose candidate functions that are polynomials of $h$ up to degree $3$ such as 
\begin{equation}\label{eq:dictionary}
\Xi(h) = \left[\mathcal P_0(h), \mathcal{P}_1(h), \mathcal{P}_2(h), \mathcal{P}_3(h)\right]\in \mathbb R^{3d+1},
\end{equation}
 where $\mathcal{P}_k(h) = \{h_1^k, \dots, h_d^k\}$ contains $k$-th degree of polynomials. Denote the coefficients in $\beta$ corresponding to the dictionary functions $\mathcal P_k$ as $\beta_{k}$ where $k = 0,1,2,3$.  As described in \cref{subsec:PCA}, the DDR method using only the linear dictionary  $\mathcal P_1(h)$ reproduces the result of PCA  with parameters $\beta_1$ as described in \cref{lem:opt_linear}. 
 Building on this result, to initialize parameters $\beta$ for an extended dictionary, we set 
   \begin{equation*}
     \beta_{1}= U D(0,\log(\varepsilon^\star))U^\ast , \qquad Q = U_k^\ast, 
 \end{equation*}
where $\varepsilon^\star$ is chosen according to  \eqref{eq:app_eps} for a given fixed $\mu$. The remaining    entries of $\beta$ are randomly initialized via the normal distribution, $[\beta]_{ij} \sim  \mathcal{N}(0,s^2)$, where we take $s=0.5$ so that its values are relatively similar to values of  $\beta_{1}$. 

\medskip

For each dataset, we train our model several times under the different conditions on entries of dictionary, initialization, hyper-parameters. 

\subsubsection*{Other methods} We compare the DDR method with PCA,  t-SNE and Umap by plotting the lower-dimensional representations of the data. Briefly, t-SNE is a nonlinear method that preserves local structures in the data by minimizing the discrepancy between pairwise similarity in the data and the pairwise similarity in the lower-dimensional embedded data (computed using the $t$-distribution) \cite{tsne}.  Umap learns a Riemannian manifold so that the data is likely to be sampled from a uniform distribution on the manifold \cite{umap}. We use the PCA and t-SNE implementations in the scikit-learn package \verb+sklearn.decomposition.PCA+ and \verb+sklearn.manifold.TSNE+. We use the Umap implementation provided in \cite{umap}.  We use the default training settings for the compared methods.

\subsection{S-shaped manifold}\label{subsec:S}
We first test our method on a synthetically generated dataset, which is in the shape of an $S$-shaped surface embedded in three dimensions. The dataset is generated by the solution at time $T=1$ to the ODE
\begin{subequations}
\label{eq:S_curve}
 \begin{align}
    \dot{z_1} &= 2z_3^3\\
    \dot{z_2} & = 0\\
    \dot{z_3} & = -2z_1^3,
    \end{align}
\end{subequations}
with $N=400$ initial conditions given by 
$$
[z_1(0), z_2(0),z_3(0)] \in  [\textrm{meshgrid}([-1:20:1],[-1:20:1]),0].
$$  
We refer to collection of points 
$x_i  = [z_1(1),z_2(1), z_3(1)]_i,$ for $ i \in [400]$
as the \textit{S-data}; a plot of the \textit{S-data} is given in the first subplot of \cref{fig:S_cubic}(a) and colored by the  first coordinate. 
 
\begin{figure}
    \includegraphics[scale = 0.5 ]{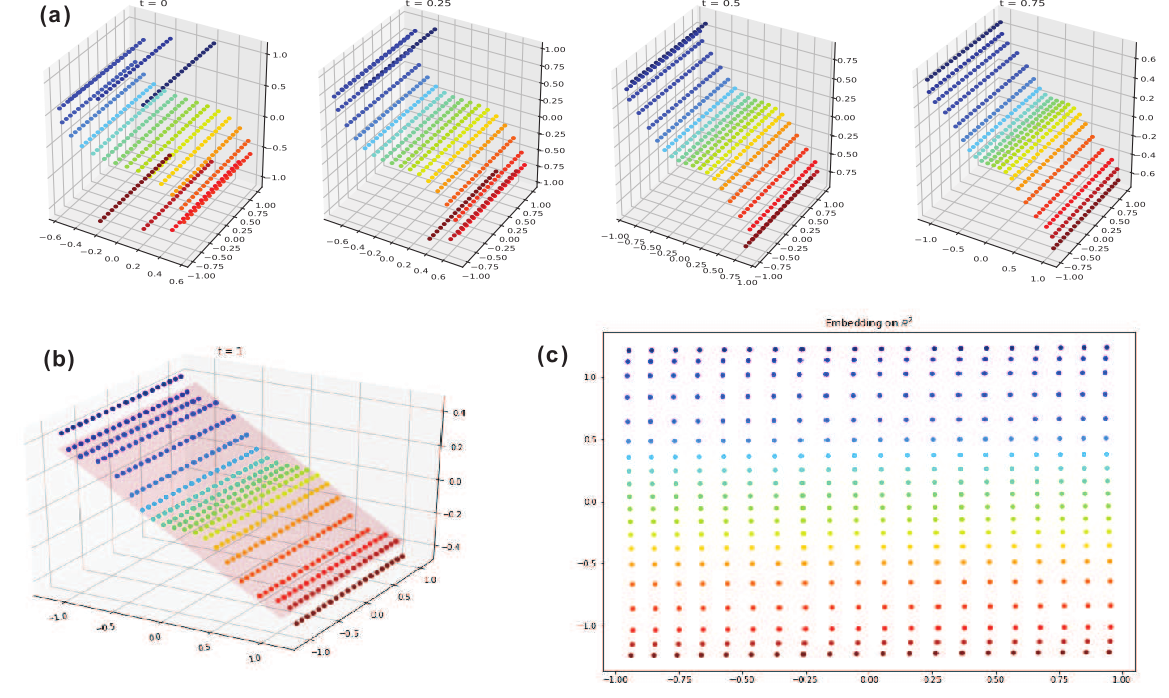}
    \caption{\textbf{(a)} We depict the positions of data over time evolving along the learned ODE by the DDR method with $\mathcal{P}_3$ in dictionary and $\mu = 0.001$. The four subfigures in (a) are the snapshots of the solution manifold at $t = 0,  0,25, 0.5$ and $0.75$. \textbf{(b)} Solution to optimal ODE at terminal $T = 1$ and $Q^\star$ subplane which is shaded red. \textbf{(c)} We plot the learned lower representations of \textit{S-data} in $\mathbb {R}^2$. See \cref{subsec:S}.}
    \label{fig:S_cubic}
\end{figure}

Since the \textit{S-data} is created by the cubic polynomial vector fields, it is natural to employ $\mathcal P_3$ functions to build a dictionary. We visualize the DDR model by plotting the evolution of the learned dynamical system and projection space. As shown in \cref{fig:S_cubic}(a), the hidden variables are initially positioned in an $S$-shaped  manifold and are gradually unfolded/flattened onto the  $Q^\star$ space over time. In \cref{fig:S_cubic}(b), we draw both hidden variables at $T$ (colored dots) and  an orthonormal subspace spanned by row vectors of $Q^\star$ (red shaded surface). The DDR method maps the \textit{S-data} to the low-dimensional representations shown in  \cref{fig:S_cubic}(c).

\subsubsection*{Hyper-parameter tuning}\label{subsec:L_curve}
The objective function of the DDR method contains two terms; a mean squared residual error ($J1$) and a kinetic energy of the data manifold traveling along the ODE ($J2$),  where a regularization hyper-parameter $\mu$ balances between $J1$ and $J2$. In practice, selecting an appropriate $\mu$ is important to reasonably train the DDR model. 
We employ the $\mathbb L$-curve criterion proposed in \cite{Lcurve} for the Tikhonov regularization hyper-parameter of the linear inverse problem. 
Denote $\beta_\mu$ and $Q_\mu$ as optimal solution to the problem
\begin{equation*}\label{mu_opt}
    \beta_\mu, Q_\mu = \argmin_{\beta,Q}  J_1 + \mu J_2,
\end{equation*}
where we use the same definition of $J_1$ and $J_2$ in \eqref{eq:regularization}. We define the curve 
\begin{equation*}\label{eq:Lcurve}
    \mathbb L =\{ J_1(\beta_\mu, Q_\mu), J_2(\beta_\mu, Q_\mu)\colon~ \mu>0\}. 
\end{equation*}
As a function of $\mu$, $J_2$ is monotonically decreasing  whereas $J_1$ is monotonically increasing. Thus the $\mathbb L$-curve has a negative slope and, in practice, takes the shape of an ``L''. Moreover, both $J_1$ and $J_2$ are equitably minimized at the elbow of $\mathbb L$-curve. In practice, we tune the regularization parameters by training the model for $\mu$ in the set {\small $ \{ 5\times 10^{-5}, 10^{-4}, 5\times 10^{-4}, 10^{-3}, 5\times 10^{-3}, 10^{-2}, 0.05, 0.1, 0.5, 1, 1.5, 2\}$} and picking $\mu$ at the elbow of the resulting $\mathbb L$-curve. 
  
We present the $\mathbb L$-curve from training the DDR method for \textit{S-data} with each $\mu$ in the above set.  \Cref{fig:L-curve}(a) shows that the vertex of $L$-curve is attained at the fourth element (numbered by $3$) in the list of $\mu,$ which is $10^{-3}$. 
In  \ref{fig:L-curve}(b), we depict the learned lower representations of the DDR method with $\mu \in  \{5\times 10^{-5}, 10^{-4}, 5\times 10^{-4}, 10^{-3} \}$. Coincided with  $\mathbb L$-curve criterion in (a), the most expected embedding is achieved with $\mu = 10^{-3}$. 

\begin{figure}
    \centering
     \includegraphics[scale=0.7]{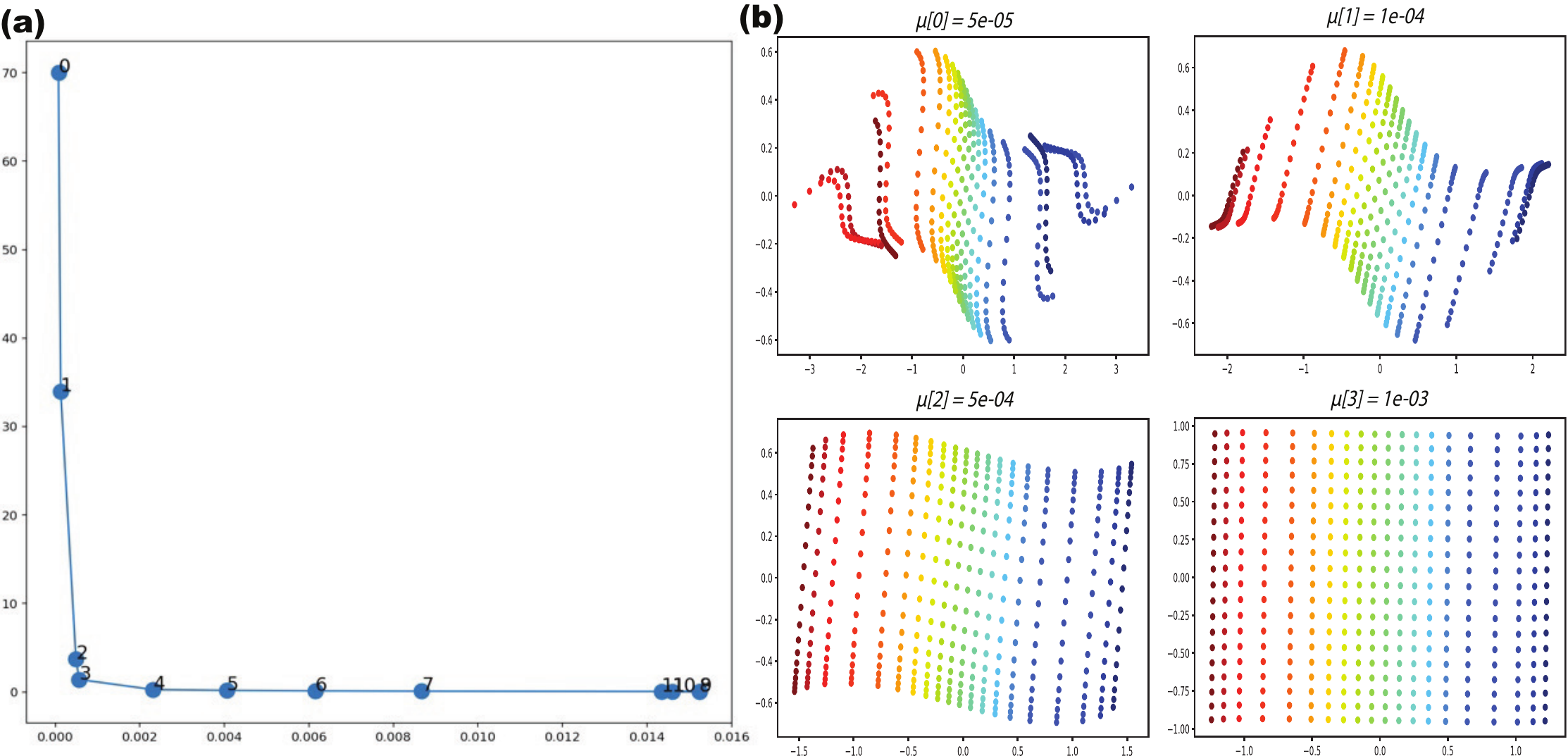}
     \caption{\textbf{(a)} Plot $\mathbb L$-curve defined as $\mathbb L=\{(J1,J2): \mu  \text{is in pre-listed set}\}$. \textbf{(b)} The lower representation of \textit{S-data} by training the DDR method under the same conditions except hyper-parameter $\mu = $\footnotesize{$\{5\times 10^{-5}, 10^{-4}, 5\times 10^{-4}, 10^{-3} \}$}.  See \cref{subsec:S}.}
    \label{fig:L-curve}
\end{figure}

\subsubsection*{Non-identifiablity}
Next, we remark that the DDR method is \textit{non-identifiable}. The \textit{S-data} could be viewed as an initial condition for the reverse ODE to \eqref{eq:S_curve} so that its solution is lying onto $span\{e_1,e_2\}$ plane at $T =1,$ where $e_i\in \mathbb R^3$ denotes a canonical basis vector whose $i$-th entry is one. It implies that the ground-truth parameters of the DDR method are exactly a coefficient of time-reversed dynamical system such that 
\begin{equation*}\label{true}
    \beta_{true}  = \footnotesize{\begin{bmatrix}
   0 &  0 &  -2\\
0 & 0 & 0\\
2 &  0 & 0\\
\end{bmatrix}},\qquad 
Q_{true} = \begin{bmatrix}
1 & 0 & 0\\
0 & 1& 0
\end{bmatrix}.
\end{equation*}
Remind that the goal of our method is finding a mapping $\mathcal{E}\colon x \mapsto y = Q h(T)$ so that only the last stage of the solution $h(T)$ should be as close as possible to $Q$ space. Hence learned vector fields and subspace may not be uniquely determined and could differ from the ground-truth. Indeed, the trained optimal parameters reported below do not agree with true parameters. 
\begin{equation*}\label{betastar}
    \beta^\star   = \footnotesize{\begin{bmatrix}
    -0.0680 &  -0.0005 &  -2.1890\\
-0.0057 & -0.0546 & -0.0287\\
0.5832 &  -0.0001 &  -0.3004\\
\end{bmatrix}},\qquad 
Q^\star  = \begin{bmatrix}
-0.9482 & -0.0139 & 0.3173\\
0.0133 & -0.9999 & -0.0040
\end{bmatrix}
\end{equation*}

\subsubsection*{Dictionary comparison}
We now consider the DDR framework with general choices of dictionaries. As formulated in \eqref{eq:dictionary}, a dictionary $\Xi$ consists of  polynomial functions of $h$ up to degree $3$. We then derive embeddings $x \xmapsto{\mathcal{E}} y$ parametrized by three cases of parameters; an initializer $(\beta_{ini},Q_{ini})$ described in  \cref{sec:experiments},  the optimizer  $(\beta_{opt},Q_{opt})$ trained by the DDR method and the ground truth $(\beta_{true},Q_{true})$ given in (\ref{true}). In \cref{fig:S_general}, the subplots (a)-(c) visualize the magnitude of all entries of each $\beta$ by varying the intensity of colors and the subplot (d)-(f) plot the resulting low representations.

\begin{figure}
    \centering
    \includegraphics[scale = 0.7]{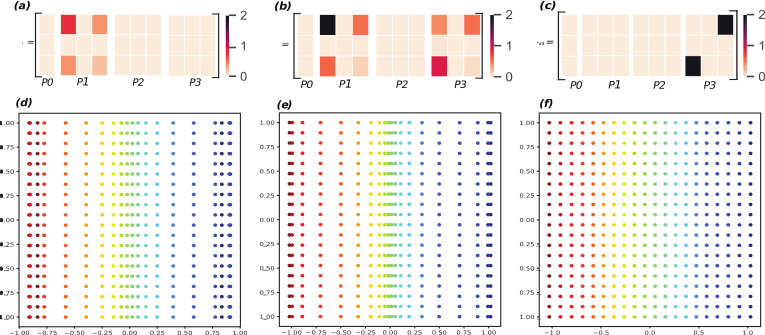}
    \caption{We represent the DDR method with a choice of dictionary with polynomials upto degree $3$ in \eqref{eq:dictionary}. \textbf{(a)-(c)} We draw heat-map about the magnitude of entries of  initial $\beta_{ini}$,  optimal $\beta_{opt}$ and ground truth $\beta_{true}$. \textbf{(d)-(f)} Plot the embedding of \textit{S-curve} in $\mathbb R^2$ generated with $(\beta_{ini},Q_{ini})$, $(\beta_{opt}, Q_{opt})$ and $(\beta_{true}, Q_{true})$ respectively. See \cref{subsec:S}.}
    \label{fig:S_general}
\end{figure}

\begin{table}
\center
\begin{tabular}{c|c|c|c}
& $(\beta_{ini},Q_{ini})$ & $(\beta_{opt}, Q_{opt})$ & $(\beta_{true}, Q_{true})$  \\
\hline
$J1$ (residual) & 0.01219 & 0.000285 & \textbf{0.00053} \\
$J2$ (regularization) & \textbf{0.00019} & 0.00069 & 0.00393 \\
$J$ (total loss) & 0.01238  &  \textbf{0.00098} &  0.00447
 \end{tabular}
 \caption{A comparison of value of objective function evaluated at  specific ($\beta$,$Q$). See \cref{subsec:S}.}
 \label{table:S_general}
 \end{table}

As pointed out in \ref{subsec:PCA}, a framework of DDR model characterized by $(\beta_{ini},Q_{ini}$) performs similarly to the PCA, where the embedding (d) formulated by an initializer is almost identical with the PCA projection. Such linear projection methods, however, couldn't capture nonlinearity in the data. As shown in Figure \ref{fig:S_general}(d),  the points located at the tail of \textit{S-data} are not recovered by any linear vector fields and are folded/overwritten on the $Q_{ini}$ space. In contrast, the DDR method encourages the underlying vector fields to be represented by nonlinear functions via the training process. In a comparison of heat maps  Figure \ref{fig:S_general}(a)-(b), the optimal coefficients corresponding to $\mathcal P_3$ being initialized by zero are activated, while the linear parts of components are still assisted. As plotted in  Figure \ref{fig:S_general}(e), the optimal lower dimensional representation perfectly rolled out $S-data$ than (d). Furthermore, we present the scores of objective functions evaluated at parameters in Table \ref{table:S_general}. By comparing the first two columns of the table, the residual error is mainly minimized, whereas the rise in regularization loss is relatively negligible. Therefore, the DDR method is established to reinforce complexity in dynamics and improve the performance of the dimension reduction mapping by minimizing a total objective function.

Next, we observe the influence of the regularization term in \eqref{eq:objective_regular} on the learning of a data manifold. Both optimal and true embedding  in \cref{fig:S_general}(e)-(f) could be considered as a good lower dimensional representation of \textit{S-data} because the  initial mesh grid is well retrieved.  As tabulated in the last column of \cref{table:S_general}, however, embedding (f) spends extensive kinetic energy of dynamics to transform the manifold. If a given manifold is forced to move by a higher speed of vector field,  then inherent properties or key structure of data could be contaminated.  Indeed, the minimum of total loss is attained at $(\beta_{opt},Q_{opt})$. Therefore, we show that the DDR method is designed to balance between projecting onto reduced dimensional space and preserving the structures of the data.

\subsubsection*{Stability}
In \cref{subsec:stability}, we prove that the DDR mapping $x_i \xmapsto{\mathcal{E}} y_i$  is stable under the noise in a given data. We numerically examine that the mapping learned with a given data is generalizable to perturbed data without retraining the model. In \cref{fig:stability}, we depict \textit{S-data} interrupted by the noise and its lower representation applied by the optimal embedding expressed by $(\beta_{opt},Q_{opt})$.  Note that four different perturbed data are created by adding a perturbation $z$ element-wise, where $z \sim  N(0,\eta^2)$ where the standard deviation of the noise $\eta$ varies in  $[0.01,0.05,0.1,0.5]$. Since the magnitude of plane  \textit{S-data} is ranged in $[-1.2,1.2]$, low dimensional representations of noisy data are reliable as long as $\eta$ is relatively small.

\begin{figure}
    \centering
    \includegraphics[scale = 0.6]{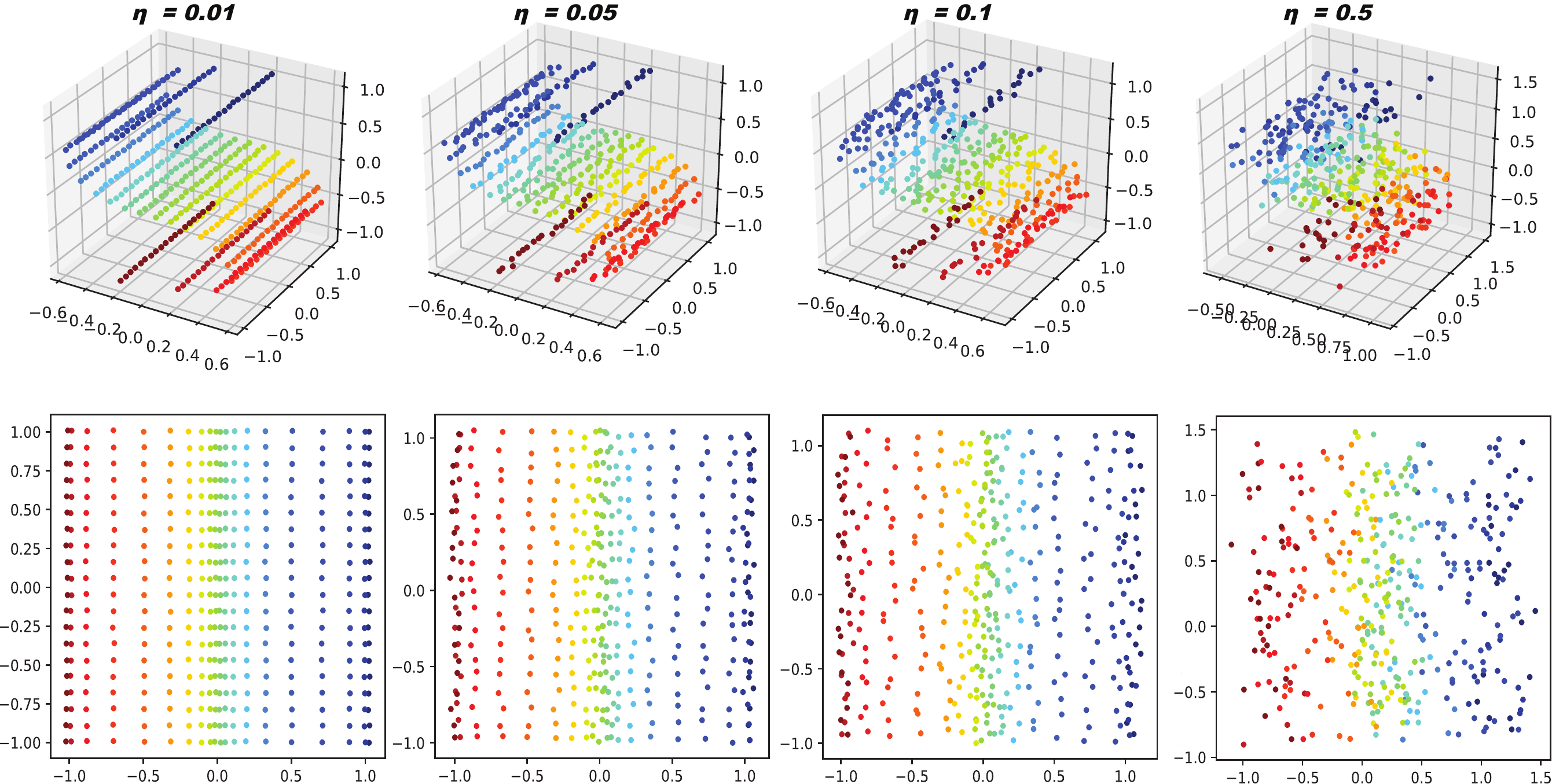}
    \caption{We apply the optimal DDR model parametrized by $(\beta_{opt},Q_{opt})$ to the \textit{S-data} that has been perturbed by random noise with standard deviation $\eta$; see \cref{subsec:S}.  }
    \label{fig:stability}
\end{figure}

\subsubsection*{Generative model for the $S$-shaped manifold.}
In \cref{s:GenModel}, we explained how the DDR method can be extended as a generative model. After training the DDS method $x \xmapsto{\mathcal{E}} y = Q^\ast h(T)$, the decoder $\mathcal{D}\colon  \mathbb R^k \to \mathbb R^d$ is defined by $\mathcal{D}(y) = h(0)$, where $h(t)$, $t\in[0,T]$ satisfies the time-reversed  dynamical system \eqref{e:TimeReverseDS}  with final condition $h(T) = Q^\ast y$. If there is zero training error, we have that $\mathcal{D}\circ \mathcal{E} = I$ on the training data. Further, in \cref{t:Decoder}, we showed that there exists a neighborhood about the embedded data, such that the decoder is well-defined. 

We further illustrate this idea using the S-shaped dataset (see \cref{fig:generate}(a)). 
We consider regularly sampled points in the latent space,  $y\in\textrm{meshgrid}\footnotesize([-1:20:1],[-1:20:1])$. 
For each $y$, we solve the time-reversed  dynamical system \eqref{e:TimeReverseDS} with initial condition $Q^\ast y$; these initial conditions are plotted in \cref{fig:generate}(b), using a triangular mesh. 
The decoded points $\mathcal{D}(y)$ are then plotted in \cref{fig:generate}(c), again using a triangular mesh. We view the map $\mathcal D \colon \mathbb R^2 \to \mathbb R^3$ as a parameterization of an approximation to the data manifold in \cref{fig:generate}(a). The approximation comes from the fact that the training error for the DDR method on this dataset is non-zero. 

\begin{figure}
    \centering
    \includegraphics[scale = 0.3]{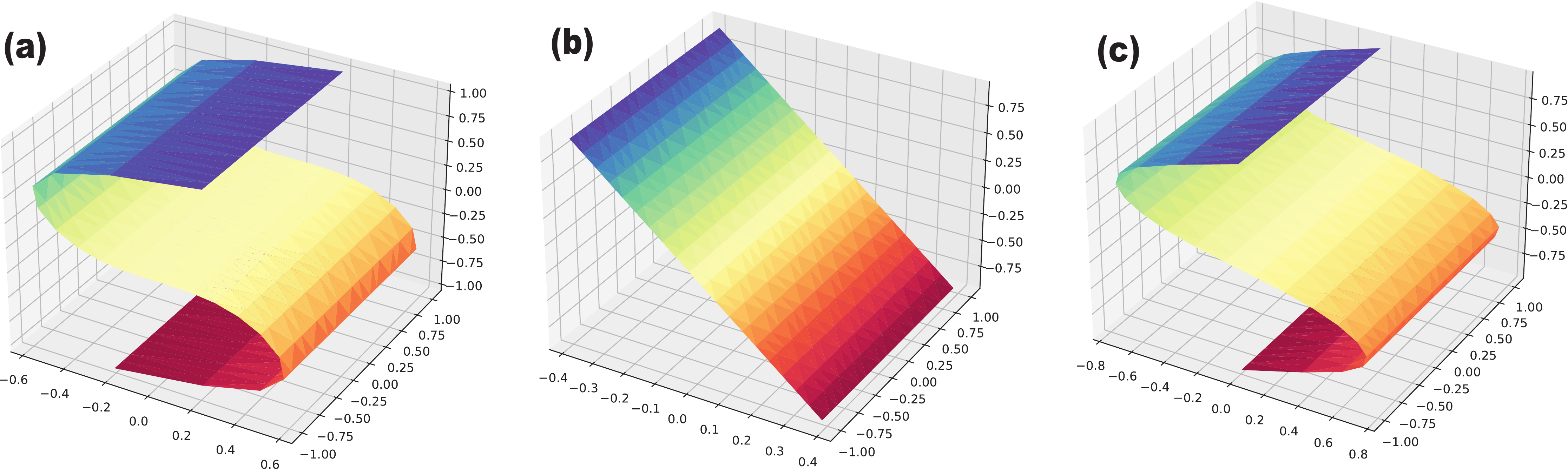}
    \caption{We extend the DDR framework for generative purposes. 
    (a) A given $S$-shaped manifold is used for training an encoder $\mathcal E $. 
    (b) We regularly sample points from the latent space and plot their image under the mapping  $y \mapsto Q^\star y$. 
    (c) Using the decoder, we plot the decoded manifold, which is an approximation to the original $S$-shaped manifold. Each of the surfaces are drawn using a triangular mesh. }
    \label{fig:generate}
\end{figure}

\subsection{Iris-data}\label{subsec:iris}
The \textit{iris} dataset contains $N=150$ instances where each data has $d=4$ features and is classified into three types of iris. The data is downloaded via \verb+sklearn.datases.load_iris()+. 
We consider embedding this $d=4$ dimensional \textit{iris} data onto $k=2$ dimensional space. For the DDR method, we conduct a hyper-parameter search using the $\mathbb L$-curve test and choose $\mu = 0.005$. 
In  \cref{fig:iris_comparison}, we plot the embedded data which are colored by their classes along with the results obtained via PCA, t-SNE, and Umap. Comparing the four methods, we observe that the DDR method clusters the data as much as the other methods. In fact,  the clustering boundary of the DDR method, especially the margin between group red and blue, is more distinct and noticeable than other methods. The nonlinear dynamics in the DDR method end up reducing the in-class variance slightly more than PCA but without collapsing the clusters as t-SNE and Umap do for this dataset. This shows that DDR method maintains both large-scaled structure and pairwise distances between dataset. 
\begin{figure}
    \centering
    \includegraphics[scale =0.6]{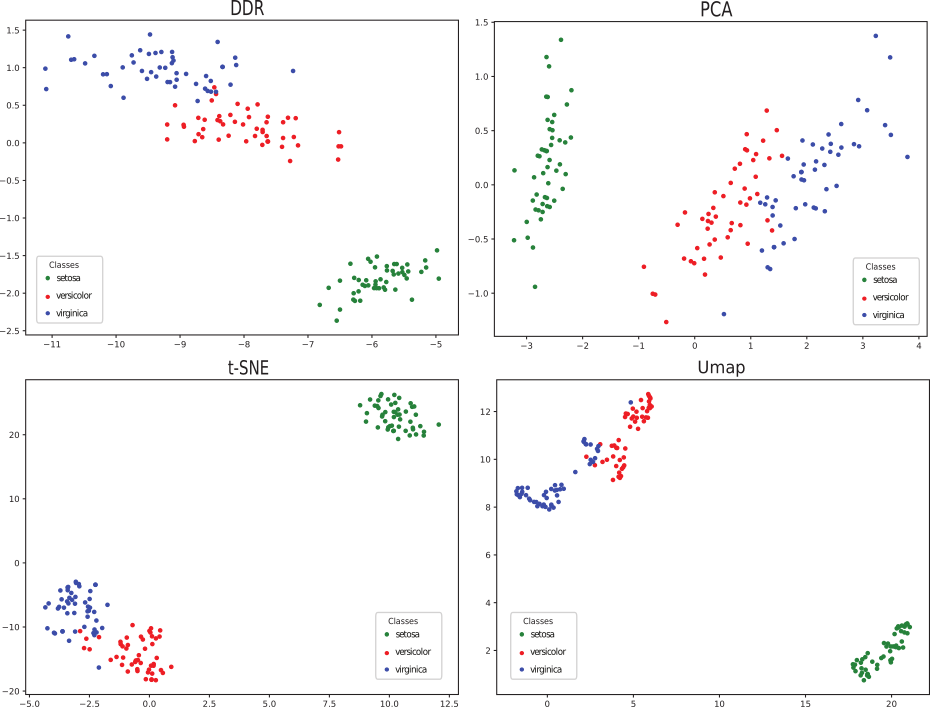}
    \caption{Comparison of projection of $iris$-data using DDR, PCA, t-SNE, and Umap. See \cref{subsec:iris}. }
    \label{fig:iris_comparison}
\end{figure}

\subsection{Handwritten digits-data}\label{subsec:digit}
The \textit{digits} data contains $8\times 8$ images of handwritten digits $0-9$. We downloaded the data from the sklearn dataset dictionary using \verb+ sklearn.datasets.load_digits()+. Note that we only use a subset of the images, digits $0-3$, so we have $N=364$ examples. 
We also normalized the data by changing the range of the pixel values from $[0,16]$ to $[0,1]$. 
To reduce computing time, we applied PCA to reduce the dimension from $64$ to $10$ dimensions. 
We examined the DDR method with extensive dictionary sweeps and hyper-parameter searching, and the best result is found with $\mathcal{P}_3$, using a random initialization, and $\mu = 0.01$. 

The resulting two-dimensional embedding obtained via the DDR method is shown in Figure \ref{fig:digit}, as well as the embeddings obtained via PCA, t-SNE, and Umap. We observe that the DDR method clusters the digits but not as strongly as t-SNE and Umap. Compared to the PCA embedding, the boundary between classes $1$, $2$, and $3$ (colored by sky blue, yellow, and brown, respectively) is better defined than PCA. This shows that the nonlinear mapping delivers more information than the linear one. However, we see a few misinterpreted instances by the DDR method (several brown dots in the yellow cloud), which may correspond to the brown island found in t-SNE and Umap subplots.

\subsubsection*{Computational time}
The training of our model depends on the initial dimension $d$ of the data and the size of the dictionary $d_n$. For, the \textit{handwritten digits} dataset, we used $d = 10$ and $d_n = 30$. 
To train the DDR model on this dataset, we used $900$ epochs taking an average of $2.2692$ seconds per epoch. In comparison,  t-SNE and Umap took less training time, $5.7168$ seconds and $16.4622$ seconds, respectively. Each of these methods used $1000$ epochs and $500$ epochs, respectively. Our implementation of the DDR method is slower than these other methods, which could be improved in future work.

\begin{figure}
    \centering
    \includegraphics[scale =0.4]{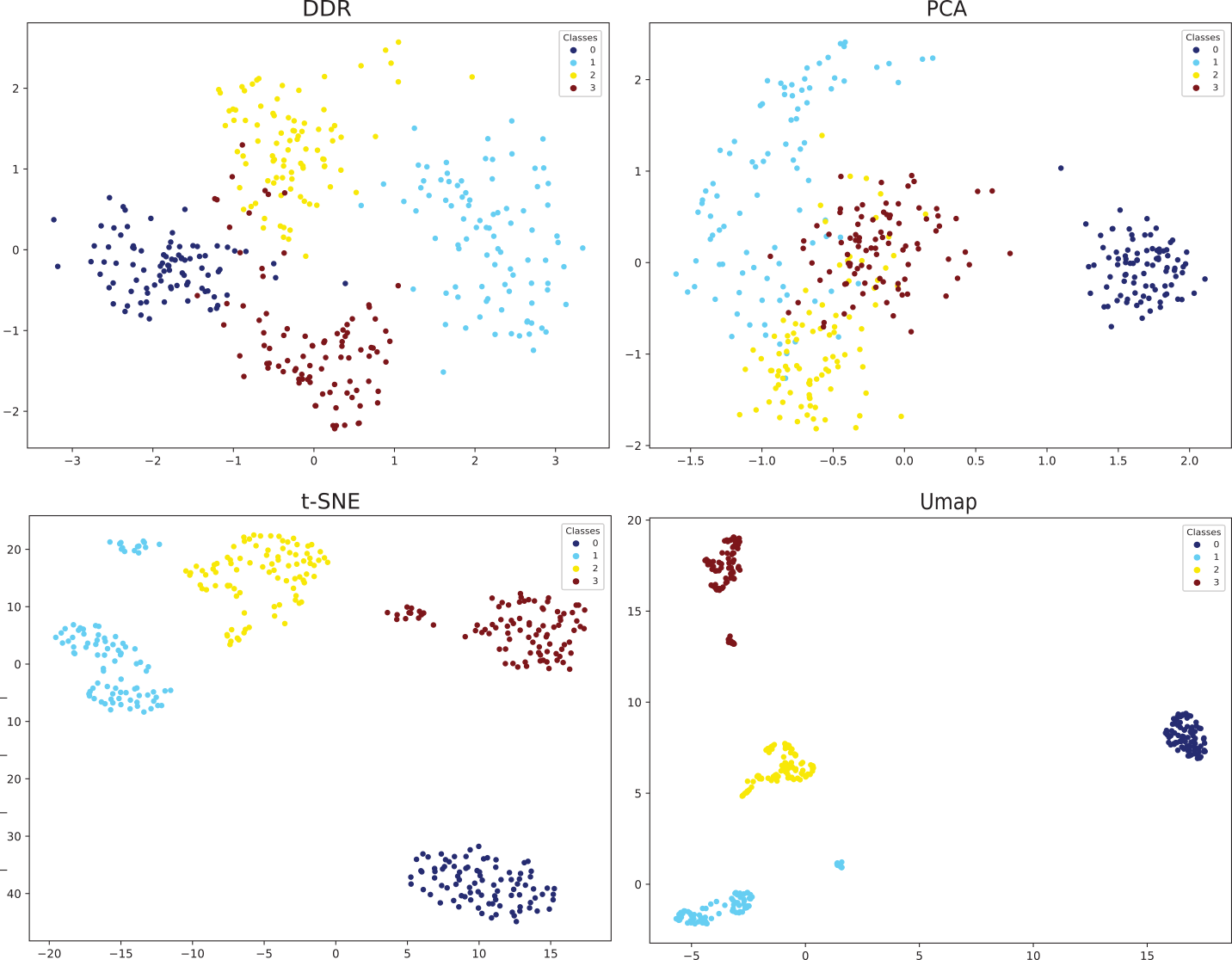}
    \caption{Comparison of embeddings for a subset of the \textit{digit}-data which consists of digits $0-3$ using DDR, PCA, t-SNE, amd Umap. See \cref{subsec:digit}.}
    \label{fig:digit}
\end{figure}

\section{Discussion} \label{s:disc}
In this work, we proposed a framework for learning a low-dimensional  representation of data based on nonlinear dynamical systems, called \emph{dynamical dimension reduction} (DDR).  In the DDR model, each point $x$ is evolved via a nonlinear flow \eqref{eq:ODE} towards a lower-dimensional subspace; the projection onto the subspace gives the low-dimensional embedding. Training the model involves identifying the nonlinear flow and the subspace. 
Following the equation discovery method, we represent the vector field that defines the flow using a linear combination of dictionary elements, where each element is a pre-specified linear/nonlinear candidate function. 
A regularization term for the average total kinetic energy is also introduced and motivated by optimal transport theory. 
We prove that the resulting optimization problem \eqref{eq:optimization} is well-posed (see \cref{thm:ex_opt}) and establish several  properties of the DDR method (see \cref{sec4}). 
We also show how the DDR method can be trained using a gradient-based optimization method, where the gradients are computed using the adjoint method from optimal control theory (see \cref{thm:adjoint}). Implementing the DDR method via \cref{alg:DDR}, we demonstrate that its performance is comparable to other dimension reduction methods including PCA, t-SNE and Umap (see in \cref{sec:experiments}). In examples, we observed that the representability/expressibility of the DDR method is improved over PCA due to the nonlinear functions in the governing vector field; to capture complex data structures, the parameters corresponding to the nonlinear dictionary elements are activated.   The t-SNE and Umap methods solely rely on local distances and PCA focuses on the global structure of the data. In contrast, the DDR method balances these objectives, minimizing not only a residual error but also the kinetic energy of the trajectories (a rate of deformation of the data manifold along the flow). 

We implemented the DDR method as a proof of concept. However, this method is slow to train because the solutions to the forward ODE for the hidden variable \eqref{eq:ODE} and the adjoint ODE \eqref{eq:adjoint} are expensive to compute.  
A natural future direction for this work is to accelerate the algorithm by using  multi-step ODE solvers and allowing the method to adaptively chose a coarser discretization. 
Furthermore, we could generalize the governing dynamical system to include non-autonomous vector fields, $\Phi$
or respect additional structure, \eg, Hamiltonian or symplectic \cite{hamiltonian,symplectic}. We could also modify the form of the dynamical system; for example, the second-order momentum equation might improve computational efficiency and long-term dependencies \cite{heavy}.

The theory of dynamical systems could be used to further prove analytical results for the DDR model. For example, while \cref{thm:stability} gives a stability result for a given DDR embedding in terms of the data, we view it as an interesting and challenging result to prove the stability of the training with respect to changes in the data as well as the consistency of the model.  
Further ideas from equation discovery could also be incorporated, such as looking for vector fields that have a sparse representation in terms of the dictionary.

\subsubsection*{Acknowledgements.} We would like to thank Harish Bhat for helpful discussions in the early stages of this work.

\printbibliography

\end{document}